\documentclass{article}

\PassOptionsToPackage{numbers,sort&compress}{natbib}
\usepackage[preprint]{neurips_2026}

\usepackage[utf8]{inputenc} % allow utf-8 input
\usepackage[T1]{fontenc}    % use 8-bit T1 fonts
\usepackage[hidelinks]{hyperref}

\usepackage{url}            % simple URL typesetting
\usepackage{booktabs}       % professional-quality tables
\usepackage{amsfonts}       % blackboard math symbols
\usepackage{nicefrac}       % compact symbols for 1/2, etc.
\usepackage{microtype}      % microtypography
\usepackage{xcolor}         % colors
\usepackage{booktabs}   % for \toprule, \midrule, \bottomrule
\usepackage{multirow}   % for \multirow
\usepackage{makecell}   % for \makecell
\usepackage{amsmath,amssymb,amsthm,graphicx}
\usepackage[ruled,vlined,linesnumbered]{algorithm2e}
\usepackage{algpseudocode}   % 或 \usepackage{algorithmic}

\usepackage{wrapfig} % for motivation png

\usepackage{xcolor}
\usepackage[table]{xcolor}

\definecolor{rowblue}{RGB}{230,242,255}

\title{CoDMD: Copula-aware Distribution Matching Distillation for Fast Video Generation}

\author{
\textbf{Wenhu Zhang}$^{1}$ \qquad
\textbf{Kun Cheng}$^{2}$\thanks{Project leader,\textsuperscript{\dag}Corresponding author} \qquad
\textbf{Changyuan Wang}$^{3}$ \qquad
\textbf{Shiyao Li}$^{5}$ \qquad
\textbf{Yuechen Zhang}$^{4}$ \\
\textbf{Wenbo Li}$^{4}$ \qquad
\textbf{Jiajun Zha}$^{1}$ \qquad
\textbf{Jingyi Zhang}$^{2}$ \qquad
\textbf{Kang Zhao}$^{2}$ \qquad
\textbf{Jiaya Jia}$^{1}$\textsuperscript{\dag} \\
$^{1}$HKUST \qquad
$^{2}$Wan Team, Alibaba Group\qquad
$^{3}$THU \qquad
$^{4}$CUHK \qquad
$^{5}$ XDU \\
Project page: \url{https://andrew-zhang98.github.io/CoDMD_page}
}

\begin{document}

% \definecolor{ourcolor}{RGB}{200,30,30}                 % red highlight for our additions

% \definecolor{ourcolor}{RGB}{255,170,102}
\definecolor{ourcolor}{RGB}{235,140,70}

\definecolor{ourhighlight}{RGB}{255,237,215}
\newcommand{\ourc}[1]{\textcolor{ourcolor}{#1}}
\newcommand{\methodname}{CoDMD\xspace}                 % short name (used everywhere)
\newcommand{\methodfull}{Copula-aware DMD\xspace}      % full name (first occurrence)

\definecolor{myred}{RGB}{200,30,30}      
\definecolor{myblue}{RGB}{0,70,160}
\definecolor{mygreen}{RGB}{0,120,0}
\definecolor{rowblue}{RGB}{230,242,255}

\newcommand{\red}[1]{\textbf{\textcolor{myred}{#1}}}
\newcommand{\blue}[1]{\textbf{\textcolor{myblue}{#1}}}
\newcommand{\green}[1]{\textbf{\textcolor{mygreen}{#1}}}

\maketitle

% \vspace{-1cm}
% \begin{center}
% \textbf{Project page:} \url{https://andrew-zhang98.github.io/CoDMD_page}
% \end{center}

\begin{abstract}
Few-step distillation for video diffusion models has attracted significant attention, driven by the urgent demand for efficient deployment in real-world scenarios. 
However, Distribution Matching Distillation (DMD), a leading paradigm, tends to degrade under limited NFE budgets, manifesting in video generation as layout instability, oversaturation, and broken motion dynamics.
We trace this failure to a structural limitation: standard DMD is an intra-sample distribution-matching objective with coordinate-wise gradients, and thus imposes no explicit constraint on the relational geometry across batch elements or temporal frames, leaving the underlying copula largely unregulated.
Combined with the mode-seeking tendency of its reverse-KL objective, this absence of relational guidance makes DMD prone to collapsing into local optima in the few-step regime.
Motivated by this insight, we propose Copula-aware DMD (CoDMD), a lightweight relational regularizer that reuses score estimates already produced by the frozen teacher and the online fake model to construct pairwise relation matrices across samples and frames.
These are matched through a supplementary distributional objective that requires no additional networks, datasets, or sampling trajectories. 
On the Wan-2.1-T2V model series at 1.3B \& 14B scales, CoDMD distills 50-step teachers into 4-step students, achieving an approximate 25$\times$ speed-up while attaining VBench scores of 84.46 \& 84.87, outperforming prior trajectory-based (rCM 82.81 \& 84.05) and distribution-based (DMD 83.38 \& 83.81) methods. 
\end{abstract}

\vspace{-0.5cm}
\begin{figure*}[h!]
  \centering
    \includegraphics[width=\linewidth]{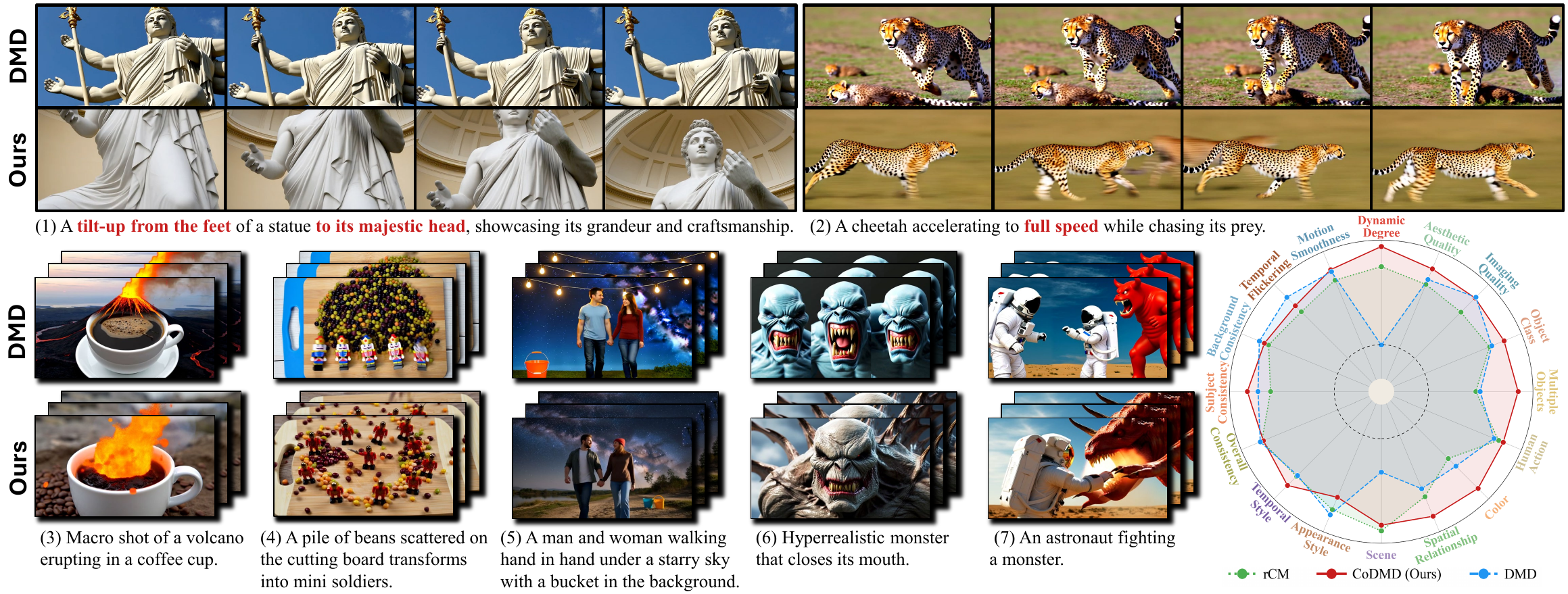}
    \vspace{-0.7cm}
  \caption{Sampled 4-step videos. DMD suffers from (1)\emph{failed camera motion}; (2)\emph{super slow action}; (3,4)\emph{broken layouts}; (5,6,7)\emph{oversaturation} and (4,7)\emph{instruction misalignment}. On the VBench radar, \red{\methodname{} (84.46)} outperforms previous methods, i.e., \blue{DMD (83.38)} and \green{rCM (82.81)}, on most of the dimensions, with the largest gains on \emph{dynamic degree}, \emph{color} and \emph{multiple objects}. Best zoomed in.}
  \label{fig:teaser}
\end{figure*}

\label{sec:intro}

\section{Introduction}
\label{sec:intro}

%-----------------------------------------------------------------------
% Paragraph 1: task + acceleration motivation
%-----------------------------------------------------------------------
Diffusion-based generative models have recently emerged as a leading paradigm for both image and video synthesis~\cite{ho2020denoising,song2020score,rombach2022high,ho2022imagen,karras2022elucidating,blattmann2023stable,wan2025wan}, driven by their favorable scaling behavior and superior perceptual fidelity. However, their practical deployment is hampered by increasing spatial resolutions, longer clip durations, and the iterative denoising process, which typically requires 50–100 forward passes per sample. To bridge this gap, a growing body of work pursues few-step distillation, which compresses the number of function evaluations (NFE) required at inference time and aims to deliver an efficient yet high-fidelity media generator.

%-----------------------------------------------------------------------
% Paragraph 2: DMD paradigm + existing composite remedies (image-only)
%-----------------------------------------------------------------------
Among existing few-step distillation paradigms, Distribution Matching Distillation (DMD)~\cite{yin2024one} aligns the student-induced distribution with that of the teacher through a distributional objective, rather than enforcing strict timestep-by-timestep trajectory consistency. When the number of inference steps is aggressively reduced (e.g., 4 steps or fewer), however, DMD exhibits a pronounced tendency toward mode collapse, manifesting in the image domain as oversaturation and reduced sample diversity.
Several recent efforts augment DMD with auxiliary training paradigms to recover the lost quality. DMD2~\cite{yin2024improved} introduces a GAN loss to sharpen details. DMDR~\cite{jiang2025distribution} combines DMD with reinforcement learning to improve overall fidelity. rCM~\cite{zheng2025large} jointly trains DMD with consistency models to mitigate mode collapse. DP-DMD~\cite{wu2026diversity} introduces a two-stage pipeline with v-prediction to pursue diversity. While these approaches partially relieve the collapse symptoms, they have primarily been designed for the general or image-domain generation. The substantially more challenging regime of few-step video distillation, in contrast, remains relatively underexplored from a DMD perspective.

%-----------------------------------------------------------------------
% Paragraph 3: video failures reveal two intertwined problems
%-----------------------------------------------------------------------
Applying DMD to video generation amplifies image-domain defects into more conspicuous failures. As shown in Fig.~\ref{fig:teaser}, DMD often exhibits (i) compromised motion fidelity (failed camera motion, overly slow actions), (ii) structural inconsistencies (broken layouts, oversaturation), and (iii) semantic misalignment with the input prompt under the 4-step budget.

%-----------------------------------------------------------------------
% Paragraph 4: Sklar-based diagnosis of the two failure sources
%-----------------------------------------------------------------------
To articulate this deficiency precisely, we invoke Sklar's theorem~\cite{sklar1959fonctions,nelsen2006introduction}, which decomposes any joint distribution into its marginal distributions and a \textbf{copula} that fully encodes their dependence structure. From this perspective, high-quality video generation requires not only realistic local appearance (the marginals) but also faithful dependence among frames, objects, layouts, and motions (the copula). Standard DMD provides distributional supervision at the level of individual generated samples, yet it imposes no explicit constraint on the pairwise relational geometry (e.g., distances or angles) across frames or across samples. Consequently, two student distributions can appear nearly indistinguishable under marginal or pointwise sample-level criteria while exhibiting substantially different temporal, spatial, or semantic couplings.

This under-constrained relational degree of freedom is particularly detrimental for video. Motion strength, camera trajectory, object interaction, and layout persistence are all expressed through copulas rather than isolated per-frame statistics. Weakly regulated copulas therefore manifest as unstable dynamics, broken layouts, over-smoothed or weakened motion, and inconsistent semantic alignment. This issue is further exacerbated by the reverse-KL objective underlying DMD: reverse-KL aggressively penalizes the student for placing probability mass in regions weakly supported by the teacher, but does not symmetrically penalize missed teacher-supported modes that the student fails to cover~\cite{wu2026diversity,lu2025adversarial,xu2025one}. Under stringent NFE budgets, where the student's expressive capacity is already constrained, this asymmetry biases optimization toward overly concentrated local optima that retain only dominant patterns while discarding weaker yet semantically important variations. Together, the lack of explicit copula constraints and the mode-seeking~\cite{bishop2006pattern} tendency of reverse-KL make standard DMD particularly brittle in the few-step video regime.

%-----------------------------------------------------------------------
% Paragraph 5: CoDMD --- lightweight relational augmentation
%-----------------------------------------------------------------------

Guided by this diagnosis, we propose \methodfull{} (\methodname), a lightweight copula-aware regularizer that augments DMD without introducing any auxiliary networks, datasets, or sampling trajectories. Our key observation is that, within every DMD iteration, the score functions of both the frozen teacher and the online fake model are already evaluated for the standard distribution-matching loss. The relational structure of the teacher distribution can therefore be extracted essentially for free. Concretely, \methodname{} reuses these score estimates to construct pairwise relation matrices at two complementary granularities: across batch elements and across temporal frames within each video. A supplementary distributional objective then aligns the student's relation matrices with the teacher's relational geometry, explicitly regularizing the structure that standard DMD leaves under-constrained. By encouraging the student to preserve the teacher's relational variability across both batch and frame, this objective discourages degenerate concentration on a narrow subset of outputs, thereby empirically counteracting the mode-seeking bias of reverse-KL.

%-----------------------------------------------------------------------
% Paragraph 6: experiments + contributions
%-----------------------------------------------------------------------
We evaluate \methodname{} on Wan-2.1-T2V diffusion model series at both 1.3B \& 14B scales. \methodname{} distills 50-step teachers into 4-step students while maintaining strong visual quality, achieving an approximate 25$\times$ lossless acceleration. On VBench, it attains total scores of 84.46 \& 84.87, outperforming trajectory-based methods such as rCM (82.81 \& 84.05) and distribution-based algorithms such as DMD (83.38 \& 83.81), with notable gains in dynamic degree (+15), color (+2.8), and multiple objects (+3.7). Our contributions are as follows:

\begin{itemize}
    \item We provide a copula perspective on why few-step DMD is fragile in video generation: intra-sample coordinate-wise supervision does not explicitly preserve the relational structure required by motion, layout, and semantics, while the reverse-KL objective further encourages overly concentrated solutions under limited NFE budgets.

    \item We propose \methodfull{} (\methodname), a lightweight copula-aware regularizer for DMD that matches teacher-student pairwise relations across batch elements and temporal frames using score estimates already computed in standard training, requiring no extra networks, data, or sampling trajectories.

    \item Across two scales of the Wan-2.1-T2V model series, \methodname{} consistently improves few-step video distillation, achieving $\sim$25$\times$ speed-up while outperforming prior trajectory-based and distribution-based methods on the VBench total score and across most dimensions.

\end{itemize}

\section{Related Work}
\label{sec:related}

\paragraph{Video diffusion and few-step acceleration.}
Modern video generators rest on diffusion and flow-matching formulations~\cite{ho2020denoising,song2020score,lipman2023flow,liu2023flow} realized by spatio-temporal transformer backbones~\cite{blattmann2023stable,wan2025wan,kong2024hunyuanvideo,yang2024cogvideox,seedance2026seedance}. They reach strong fidelity, but require tens to hundreds of denoising steps per clip, which is prohibitive for interactive deployment. Training-free samplers such as DDIM~\cite{song2020denoising} and DPM-Solver~\cite{lu2022dpm,zhang2022fast} shorten the trajectory but still need many steps to remain visually satisfactory. Training-based distillation falls into two families: trajectory-matching methods that align the student along the sampling path~\cite{salimans2022progressive,liu2023instaflow,esser2024scaling,song2023consistency,song2023improved,lu2024simplifying}, and distribution-matching methods that align student and teacher at the data-distribution level. Our work builds on the latter.

\paragraph{Distribution Matching Distillation and its remedies.}
DMD~\cite{yin2024one,yin2024improved} updates the student via a score-as-target surrogate whose gradient is given by the score difference between a frozen teacher and a learnable fake model~\cite{wang2023prolificdreamer}. DMD preserves fine-grained details well in single-image distillation, but its naive few-step use often suffers from oversaturation and mode collapse~\cite{lu2024simplifying,yin2024one,yin2024improved}. These issues become more severe in video generation, where motion weakens, layouts become unstable, and sampled clips grow mutually similar. Recent remedies improve DMD by combining it with additional paradigms, such as distribution-shaping rewards~\cite{jiang2025distribution}, consistency training~\cite{zheng2025large,song2023consistency}, and two-stage pipelines~\cite{wu2026diversity}. While effective, these approaches add extra models, trajectories, or training stages on top of DMD, leaving the underlying coordinate-wise score objective largely unchanged. In contrast, \methodname{} strengthens the DMD objective itself rather than wrapping it with auxiliary machinery.

A parallel line of work extends DMD-style training to autoregressive long-video generation~\cite{yin2025slow,huang2025self,yang2025longlive}. These methods are largely orthogonal to our few-step setting, but their emphasis on temporal error accumulation further highlights the importance of improving the underlying DMD objective, since weak supervision can compound over long horizons.

\paragraph{Relational supervision under computation budgets.}
Relational supervision has long been recognized as a powerful tool for knowledge transfer especially under strict computation budgets, as it captures the structural geometry, e.g., pairwise distances, angles, and affinity matrices between samples rather than just individual predictions \cite{tung2019similarity,peng2019correlation}. Relational Knowledge Distillation~\cite{park2019relational} matches pairwise distances and triplet angles between teacher and student features. Relation Networks~\cite{sung2018learning} and similarity-preserving distillation~\cite{tung2019similarity} align sample affinity matrices. Contrastive representation distillation~\cite{tian2019contrastive} transfers structure through mutual-information bounds.
These foundational works motivate our use of similarity matrices as effective carriers of relational information. 
While, under the NFE-constrained setting in diffusion, the conventional single-teacher paradigm with feature-space matching becomes inadequate. Generative score distillation is intrinsically a triplet-network problem, driven by the dynamic residual gap between a frozen teacher and a fake distribution. Furthermore, to seamlessly integrate with DMD, the supervision must be adapted from the conventional feature space to the score distribution space. \methodname{} addresses this by lifting the teacher--fake residual from point-wise coordinates to pairwise copula-aware relations. To the best of our knowledge, we are the first to introduce relational preserve insight into the score distillation paradigms.

\begin{figure*}[t]
  \centering
    \includegraphics[width=\linewidth]{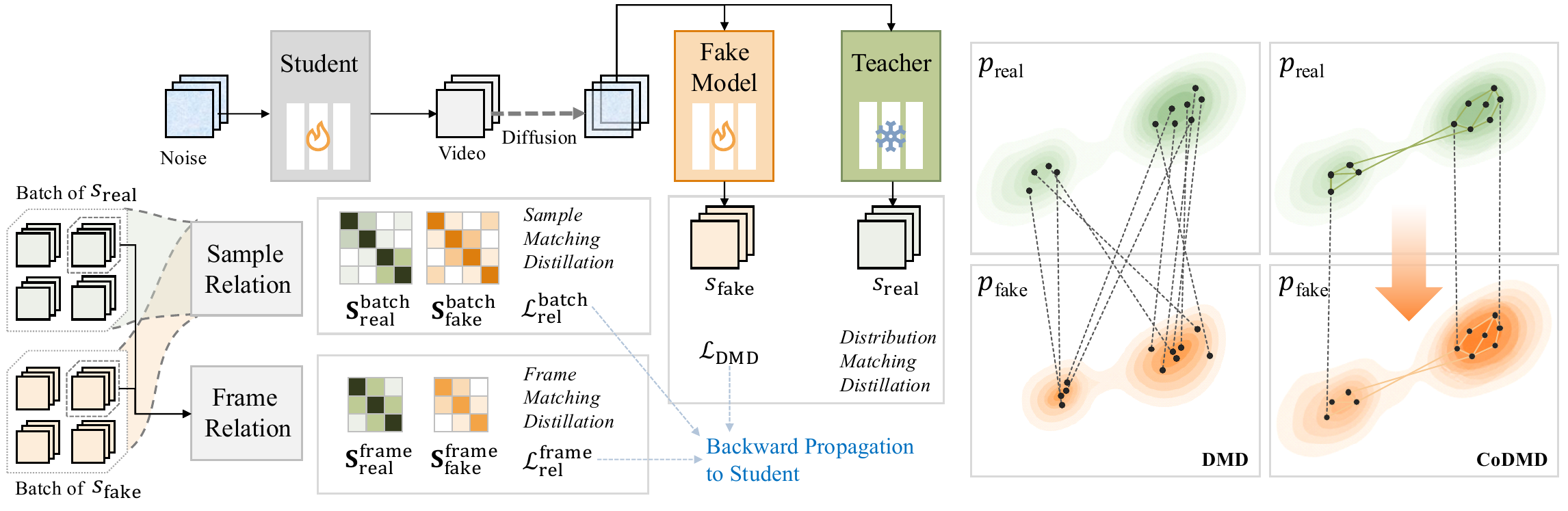}
  \caption{\emph{Left:} Training pipeline of \methodname{}. Beyond standard $\mathcal{L}_\text{DMD}$, we reuse the real and fake scores to build relational matrices at batch and frame granularities, yielding two supplementary copula losses $\mathcal{L}_\text{rel}^\text{batch}$ and $\mathcal{L}_\text{rel}^\text{frame}$.
    \emph{Right:} DMD aligns fake elements to $p_\text{real}$ but scrambles their pairwise geometry (crossed edges). \methodname{} additionally preserves the real structure among these elements.}

  \label{fig:method_overview}
\end{figure*}
\section{Method}
\label{sec:method}

We first review DMD as a \emph{score-as-target surrogate} (\S\ref{subsec:prelim}), then diagnose why this coordinate-wise recipe leaves the relational geometry of the teacher's score field less constrained (\S\ref{subsec:copula_theory}). Guided by this analysis, we propose \methodfull{} ((\S\ref{subsec:codmd})), which lifts the surrogate from coordinates to pairwise similarity matrices at two complementary granularities. Figure~\ref{fig:method_overview} gives an overview.

%-----------------------------------------------------------------------

\subsection{Preliminaries: DMD as a coordinate-wise score-as-target surrogate}
\label{subsec:prelim}

\paragraph{Notation.}
Let $\mathbf{x}_0 \sim p_\text{real}$ denote a clean video and $q_t(\mathbf{x}_t \mid \mathbf{x}_0) = \mathcal N(\alpha_t \mathbf{x}_0, \sigma_t^2 \mathbf I)$ a variance-preserving Gaussian diffusion process indexed by $t \in [0, 1]$. A pretrained teacher denoiser $\mu_\text{real}(\mathbf{x}_t, t)$ provides the perturbed score $s_\text{real}(\mathbf{x}_t, t) = -(\mathbf{x}_t - \alpha_t \mu_\text{real})/\sigma_t^2$. A few-step student generator $G_\theta : \mathbb R^{d_z} \to \mathbb R^{F \times C \times H \times W}$ produces $\mathbf{x} = G_\theta(\mathbf{z})$ from $\mathbf{z} \sim \mathcal N(\mathbf 0, \mathbf I)$, and an auxiliary denoiser $\mu_\text{fake}^\phi(\mathbf{x}_t, t)$ is trained online with $\mathcal L_\text{denoise}^\phi = \| \mu_\text{fake}^\phi(\mathbf{x}_t, t) - \mathbf{x} \|_2^2$ to track the outputs of students and provide the corresponding fake score $s_\text{fake}$.

\paragraph{Distribution-matching gradient.}
DMD minimizes the reverse-KL divergence $D_{\mathrm{KL}}(p_\text{fake} \,\|\, p_\text{real})$, whose gradient admits an expectation over coordinate-wise score differences~\cite{wang2023prolificdreamer}:
\begin{align}\label{eq:dmd_grad}
\nabla_\theta D_{\mathrm{KL}}
&\simeq \mathbb E_{\mathbf{z}, t} \Bigl[
    w_t \alpha_t \bigl( s_\text{fake}(\mathbf{x}_t, t) - s_\text{real}(\mathbf{x}_t, t) \bigr)
    \cdot \frac{\partial G_\theta(\mathbf{z})}{\partial \theta}
\Bigr],
\end{align}
where $\mathbf{x}_t$ denotes the diffused student output and $w_t$ represents a noise-level weighting.

\paragraph{Score-as-target surrogate.}
In practice, Eq.~\eqref{eq:dmd_grad} is realized through a reconstruction loss that avoids explicitly computing the student Jacobian during the forward pass. Define the update direction
\begin{equation}\label{eq:dmd_delta}
\boldsymbol{\Delta}(\mathbf{x}) \triangleq w_t \alpha_t
\bigl( s_\text{fake}(\mathbf{x}_t, t) - s_\text{real}(\mathbf{x}_t, t) \bigr),
\end{equation}
and the \emph{score-as-target surrogate} loss
\begin{equation}\label{eq:dmd_surrogate}
\mathcal L_{\mathrm{DMD}}
\;=\;
\frac12 \; \bigl\|
    \mathbf{x}
    \;-\;
    \operatorname{stopgrad}\!\bigl( \mathbf{x} - \boldsymbol{\Delta}(\mathbf{x}) \bigr)
\bigr\|_2^2.
\end{equation}
Since the target is treated as a constant, $\nabla_\theta \mathcal L_{\mathrm{DMD}} = \boldsymbol{\Delta}(\mathbf{x}) \cdot \partial_\theta G_\theta(\mathbf{z})$ exactly recovers Eq.~\eqref{eq:dmd_grad}: the KL gradient is realized by nudging $\mathbf{x}$ toward $\mathbf{x} - \boldsymbol{\Delta}(\mathbf{x})$. Crucially, $\boldsymbol{\Delta}$ is evaluated \emph{coordinate-wise}: at every spatial--temporal location it pulls the student's marginal density toward $p_\text{real}$ in isolation, with no term comparing values across locations.

%-----------------------------------------------------------------------

\subsection{Why coordinate-wise supervision is insufficient: a copula perspective}
\label{subsec:copula_theory}

We now formalize why the coordinate-wise structure of Eq.~\eqref{eq:dmd_surrogate} is consequential for video. Our argument proceeds in three steps: (i) Sklar's theorem isolates a \emph{copula score} that DMD does not explicitly supervise; (ii) reverse-KL amplifies this gap into mode collapse under the few-step constraint; (iii) a pairwise relational statistic admits a closed-form penalty for the resulting collapse.

\paragraph{Sklar's theorem and the marginal--copula factorization.}
Sklar's theorem~\cite{sklar1959fonctions,nelsen2006introduction} states that any continuous joint $p(\mathbf{y})$ on $\mathbb R^{D}$ admits a unique decomposition
\begin{equation}\label{eq:sklar}
p(\mathbf{y})
=
\Bigl( \prod_{i=1}^{D} f^{(i)}(y_i) \Bigr)
\; \cdot \;
c\!\bigl( F^{(1)}(y_1), \dots, F^{(D)}(y_D) \bigr),
\end{equation}

where $f^{(i)}, F^{(i)}$ are the marginal PDF (Probability Density Function) and CDF (Cumulative Distribution Function) of coordinate $i$, and $c : [0,1]^{D} \to \mathbb R_{\ge 0}$ is the copula density that encodes all cross-coordinate dependence and its underlying structure is invariant to coordinate-wise transformations of the marginals. Taking $\nabla \log p$ yields

\begin{equation}\label{eq:score_split}
\nabla_{\mathbf{y}} \log p(\mathbf{y})
=
\underbrace{\left[ \frac{\partial}{\partial y_i} \log f^{(i)}(y_i) \right]_{i=1}^{D}}_{=:\, s^{\text{marg}}(\mathbf{y})}
\;+\;
\underbrace{\nabla_{\mathbf{y}} \log c\!\bigl(F^{(1)}(y_1), \dots, F^{(D)}(y_D)\bigr)}_{=:\, s^{\text{cop}}(\mathbf{y})}.
\end{equation}

While DMD implicitly targets the joint score, its sample-wise coordinate-wise objective treats the regression of elements independently. This lacks explicit high-order relational constraints to directly supervise the complex cross-coordinate relations encoded in $s^{\text{cop}}$, making the structural knowledge transfer less efficient.

\paragraph{Marginal matching alone does not identify relational structure.}
A two-frame video $(X_1,X_2)$ with $U \sim \mathcal{N}(0,1)$ and
$P:(X_1,X_2)=(U,U)$, $Q:(X_1,X_2)=(U,-U)$
has identical per-frame marginals under $P$ and $Q$, yet $\mathbb{E}_P[X_1 X_2]=1$ versus $\mathbb{E}_Q[X_1 X_2]=-1$. Any objective that depends only on per-coordinate distributions assigns the same value to $P$ and $Q$, and is blind to which temporal relation the student learns (Appendix~\ref{sec:theory}). For video, this non-identifiability propagates to the failure cases in \autoref{fig:teaser}.

\paragraph{Reverse-KL amplifies the gap via mode-seeking behavior.}
DMD minimizes the \emph{reverse-}KL divergence, which is inherently mode-seeking~\cite{bishop2006pattern,xu2025one}. For a teacher mixture $P = \sum_{k=1}^{K} \pi_k P_k$ with disjoint supports and a student restricted to single-component approximants $\{P_1,\dots,P_K\}$,
$D_{\mathrm{KL}}(P_j \,\|\, P) = -\log \pi_j$,
so the optimum is the component with the largest mixture weight (Appendix~\ref{sec:theory}). More generally, $D_{\mathrm{KL}}(Q_j \,\|\, P) = D_{\mathrm{KL}}(Q_j \,\|\, P_j) - \log \pi_j$ for any $Q_j$ on component $j$: reverse-KL trades fit against weight but never penalizes the absence of other modes. In the few-step regime, a small NFE budget restricts the reachable student family, making this collapse directly actionable: the generator covers one dominant mode and drops others.

\paragraph{Relational matching penalizes collapsed solutions.}
Let $d(z,z')$ be a symmetric distance in the score space and $\mathcal{D}(Q) = \mathbb{E}_{Z,Z' \sim Q}[d(Z,Z')]$. If the teacher score distribution $P$ has modes with intra-mode distance $\le \epsilon$ and cross-mode distance $\ge \Delta$, then for any single-mode student $Q = P_j$,
\begin{equation}\label{eq:relational_gap_lb}
\mathcal{D}(P) - \mathcal{D}(Q)
\;\ge\;
\Delta\Bigl(1 - \sum\nolimits_{k=1}^{K} \pi_k^2\Bigr) - \epsilon.
\end{equation}
The factor $(1 - \sum_k \pi_k^2)$ measures the teacher's relational diversity, approaching~1 for many well-separated modes. A single-mode student therefore incurs a discrepancy that reverse-KL ignores but a pairwise objective directly measures. Eq.~\eqref{eq:relational_gap_lb} is the quantitative motivation for \methodname{}: an explicit relational regularizer assigns a non-trivial cost to the loss of cross-mode geometry that DMD alone leaves free. The next section realizes this signal within the existing DMD workflow.

%-----------------------------------------------------------------------

% ========================================================================
% Training algorithm: DMD + CoDMD with batch- and frame-level relational losses.
% Our additions are highlighted with \ourc{}, defined in the preamble as
%   \newcommand{\ourc}[1]{{\color{blue}#1}}
% ========================================================================
\begin{algorithm}
    \caption{\label{alg:ourmethod} \methodname{} training procedure.}
    \small
    \KwIn{
        Pretrained teacher $\mu_\text{real}$;
        denoising step list $\mathcal{T}=\{t_1,\dots,t_K\}$;
        % generator update ratio $r$;
        relational weights $\lambda_\text{batch},\lambda_\text{frame}$.
    }
    \KwOut{Trained few-step generator $G_\theta$.}
    
    $G_\theta \leftarrow \text{copyWeights}(\mu_\text{real}),\;
     \mu_\text{fake}^\phi \leftarrow \text{copyWeights}(\mu_\text{real})$ \tcp{Initialization}

    \While{not converged}{
        $\mathbf{x} \leftarrow G_\theta\bigl(\mathbf{z} \sim \mathcal{N}(\mathbf{0},\mathbf{I}),\;\mathcal{T}\bigr)$
            \tcp*{few-step backward simulation}

        \If{ \emph{is generator update step}}{
            $\mathcal{L}_{\mathrm{DMD}} \leftarrow \textsc{DMDLoss}(\mu_\text{real}, \mu_\text{fake}^\phi, \mathbf{x})$
                \tcp*{Eq.~\eqref{eq:dmd_surrogate}}
            \ourc{$\mathcal{L}_\text{rel}^\text{batch} \leftarrow \mathcal{L}_\text{rel}^{general}(\mu_\text{real}, \mu_\text{fake}^\phi, \mathbf{x},\, \text{batch})$}
                \tcp*{\ourc{Eq.~\eqref{eq:rel_batch}}}
            \ourc{ $\mathcal{L}_\text{rel}^\text{frame} \leftarrow \mathcal{L}_\text{rel}^{general}(\mu_\text{real}, \mu_\text{fake}^\phi, \mathbf{x},\, \text{frame})$ }
                \tcp*{\ourc{Eq.~\eqref{eq:rel_frame}}}
            $\mathcal{L}_\text{total} \leftarrow \mathcal{L}_{\mathrm{DMD}} + \ourc{\lambda_\text{batch}\,\mathcal{L}_\text{rel}^\text{batch} + \lambda_\text{frame}\,\mathcal{L}_\text{rel}^\text{frame}}$
                \tcp*{Eq.~\eqref{eq:total_obj}}
            $G_\theta \leftarrow \text{update}(G_\theta,\, \mathcal{L}_\text{total})$
        }

        Sample $t \sim \mathcal{U}(0,1)$;\;
        $\mathbf{x}_t \leftarrow \text{forwardDiffusion}(\text{stopgrad}(\mathbf{x}), t)$\;
        $\mathcal{L}_\text{denoise}^\phi \leftarrow
            \| \mu_\text{fake}^\phi(\mathbf{x}_t, t) - \text{stopgrad}(\mathbf{x}) \|_2^2$
            \tcp*{update fake score every step}
        $\mu_\text{fake}^\phi \leftarrow \text{update}(\mu_\text{fake}^\phi, \mathcal{L}_\text{denoise}^\phi)$
    }
\end{algorithm}
\vspace{-0.2pt}

\subsection{\methodfull{} (\methodname)}
\label{subsec:codmd}
The analysis above motivates a regularizer based on pairwise statistics. We instantiate it with \emph{cosine similarity matrices} computed directly from the score-based denoiser predictions. 
% These matrices could capture pairwise relational geometry, are less invariant to per-frame or per-sample scale, and serve as a tractable proxy for the copula structure analyzed in \S\ref{subsec:copula_theory}. 
% We then lift the DMD score-as-target recipe from coordinates to these matrices, following the same three-step pattern as \S\ref{subsec:prelim}: form a relational analogue of $\mathbf{x}$, compute a relational analogue of the score difference, and realize the gradient via a reconstruction loss with a stop-gradient target.

\paragraph{Relational analogue of $\mathbf{x}$: the similarity matrix.}
Every DMD training step evaluates three quantities that carry distribution-level information: the student output $\mathbf{x} = G_\theta(\mathbf{z})$, the teacher prediction $\mu_\text{real}(\mathbf{x}_t, t)$, and the fake model prediction $\mu_\text{fake}^\phi(\mathbf{x}_t, t)$. We construct pairwise cosine similarity matrices from these score-space outputs at two granularities.

At the \emph{batch level}, each generated sample $\mathbf{x}^{(b)} \in \mathbb R^{F \times C \times H \times W}$ is averaged over its temporal and spatial dimensions ($F, H, W$), preserving the channel dimension to form a $C$-dimensional vector. This yields $B$ vectors per quantity. After gathering across all GPUs to $B_\text{g}$ vectors, we compute the $B_\text{g} \times B_\text{g}$ cosine similarity matrix, where pairwise similarities are measured along the $C$ channels. At the \emph{frame level}, only the spatial dimensions ($H, W$) of each sample are averaged per frame, while the channel dimension remains intact, yielding $F$ vectors of size $C$ per sample. An $F \times F$ cosine similarity matrix is then computed per sample to capture inter-frame relations. As shown in Fig.~\ref{fig:method_overview}, we denote the resulting similarity matrices as $\mathbf{S}_\text{stu}^\text{batch}$, $\mathbf{S}_\text{real}^\text{batch}$, $\mathbf{S}_\text{fake}^\text{batch}$ (batch level) and $\mathbf{S}_\text{stu}^{\text{frame}}$, $\mathbf{S}_\text{real}^{\text{frame}}$, $\mathbf{S}_\text{fake}^{\text{frame}}$ (frame level, per sample $b$), built respectively from $\mathbf{x}$, $ s_\text{fake}(\mathbf{x}_t, t)$, and $s_\text{real}(\mathbf{x}_t, t) $.

\paragraph{Relational analogue of the score difference.}
In standard DMD, the perturbation $\boldsymbol{\Delta}(\mathbf{x})$ is the coordinate-wise score difference $s_\text{fake} - s_\text{real}$. Its relational counterpart is simply
\begin{equation}\label{eq:rel_delta}
\boldsymbol{\Delta}_{\!\mathbf{S}}
\;\triangleq\;
\mathbf{S}_\text{fake} - \mathbf{S}_\text{real},
\end{equation}
which measures how the fake-model's pairwise relations deviate from those of the teacher. This difference is computed separately for the batch-level and frame-level matrices.

\paragraph{Relational score-as-target surrogate.}
We form a detached target by nudging $\mathbf{S}_\text{stu}$ along $-\boldsymbol{\Delta}_{\!\mathbf{S}}$ and define the relational loss as the KL divergence between the softmax distributions:
\begin{equation}\label{eq:rel_loss_general}
\mathcal{L}_\text{rel}^{general}
\;\triangleq\;
\operatorname{KL}\!\bigl( \mathbf{P}_\text{tgt} \;\big\|\; \mathbf{P}_\text{stu} \bigr),
\quad\text{with}\quad
\mathbf{P}_\text{tgt} \triangleq \operatorname{softmax}\!\bigl( \tau^{-1} [\mathbf{S}_\text{stu} - \boldsymbol{\Delta}_{\!\mathbf{S}}] \bigr),
\;\;
\mathbf{P}_\text{stu} \triangleq \operatorname{softmax}\!\bigl( \tau^{-1} \mathbf{S}_\text{stu} \bigr),
\end{equation}
where $\tau > 0$ is a temperature and the gradient through $\mathbf{P}_\text{tgt}$ is stopped. A first-order expansion in $\tau^{-1}$ shows that the gradient of Eq.~\eqref{eq:rel_loss_general} with respect to $\mathbf{S}_\text{stu}$ is proportional to $\boldsymbol{\Delta}_{\!\mathbf{S}} = \mathbf{S}_\text{fake} - \mathbf{S}_\text{real}$ (see Appendix~\ref{app:rel_grad_derivation}). Chaining through $\mathbf{S}_\text{stu}$ and back to $G_\theta$ yields a gradient of the same form as Eq.~\eqref{eq:dmd_grad}, with the coordinate-wise difference $(s_\text{fake} - s_\text{real})$ replaced by the relational difference $(\mathbf{S}_\text{fake} - \mathbf{S}_\text{real})$.

We instantiate the above construction at two granularities that correspond to the two under-constrained axes identified in \S\ref{subsec:copula_theory}.

Applying Eq.~\eqref{eq:rel_loss_general} to the batch-level similarity matrices over $B_\text{g}$ videos gives
\begin{equation}\label{eq:rel_batch}
\mathcal L_\text{rel}^\text{batch}
\;=\;
\operatorname{KL}\!\bigl(
  \mathbf{P}_\text{tgt}^\text{batch} \;\big\|\;
  \mathbf{P}_\text{stu}^\text{batch}
\bigr).
\end{equation}
This term injects batch-level relational supervision that is absent from the coordinate-wise DMD gradient. By aligning the pairwise relational geometry of the student to the teacher--critic residual on the same batch, it constrains how the student's score field varies across the samples drawn in each step---the axis that the coordinate-wise gradient leaves free.

Applying the same construction to the per-sample frame-level matrices over $F$ frames gives
\begin{equation}\label{eq:rel_frame}
\mathcal L_\text{rel}^\text{frame}
\;=\;
\frac{1}{B} \sum_{b=1}^{B}
\operatorname{KL}\!\bigl(
  \mathbf{P}_\text{tgt}^{\text{frame},(b)} \;\big\|\;
  \mathbf{P}_\text{stu}^{\text{frame},(b)}
\bigr).
\end{equation}
No cross-GPU communication is required for this term. It directly counters the motion degradation diagnosed in \S\ref{subsec:copula_theory} by enforcing that temporal similarity patterns follow the teacher--fake residual.

\paragraph{Final objective.}
As standard DMD, the fake model $\mu_\text{fake}^\phi$ is trained with the denoising loss $\mathcal L_\text{denoise}^\phi$ to track the moving generator distribution. For the generator $G_\theta$, the final objective is
\begin{equation}\label{eq:total_obj}
\mathcal L_\text{total}
\;=\;
\mathcal L_{\mathrm{DMD}}
\;+\;
\lambda_\text{batch} \, \mathcal L_\text{rel}^\text{batch}
\;+\;
\lambda_\text{frame} \, \mathcal L_\text{rel}^\text{frame}.
\end{equation}
where $\mathcal L_{\mathrm{DMD}}$ denotes the coordinate-wise DMD surrogate in Eq.~\eqref{eq:dmd_surrogate}, and $\lambda_\text{batch}, \lambda_\text{frame} \ge 0$ are weighting hyperparameters. Algorithm~\ref{alg:ourmethod} summarizes the full training procedure. Since the similarity matrices are introduced only during training, inference remains identical to that of the original few-step generator.

\begin{table*}[t]
  \centering
  \scriptsize
  \caption{VBench results for Wan-2.1-T2V on 480p resolution. A full per-attribute breakdown is provided in Appendix~\ref{app:vbench_detail}, where \methodname{} leads on the majority of individual VBench dimensions.}
  \label{tab:vbench}
  \setlength{\tabcolsep}{8.5pt}
  \begin{tabular}{lccc|cc|c}
    \toprule
    \textbf{Model} & \textbf{Resolution} & \textbf{NFE} &
    \textbf{\makecell[c]{TimeCost}$\downarrow$ } & % #(s$\downarrow$)
    \textbf{\makecell[c]{Quality\\Score$\uparrow$}} &
    \textbf{\makecell[c]{Semantic\\Score$\uparrow$}} &
    \textbf{\makecell[c]{Total\\Score$\uparrow$}} \\
    \midrule
    % ------------------------------------------------------------------
    \multicolumn{7}{c}{\textit{Wan2.1-T2V-1.3B}} \\
    \midrule
    Teacher (Pretrained) & $832\!\times\!480\!\times\!81$ & $50{\times}2$ & 270s & 84.43 & 80.73 & 83.69 \\
    \midrule
    DMD~\cite{yin2024one} & $832\!\times\!480\!\times\!81$ & 4 & 11s & 84.53 & 78.78 & 83.38 \\
    DCM~\cite{lv2025dual} & $832\!\times\!480\!\times\!81$ & 6 & 18s & 83.12 & 74.33 & 81.36 \\
    sCM~\cite{lu2024simplifying} & $832\!\times\!480\!\times\!81$ & 6 & 18s & 83.42  & 72.98 & 81.33   \\
    rCM~\cite{zheng2025large} & $832\!\times\!480\!\times\!81$ & 4 & 11s & 83.72 & 79.18 & 82.81 \\
    AVD~\cite{you2026adaptive} & $832\!\times\!480\!\times\!81$ & 4 & 11s & 84.73 & 79.80 & 83.75 \\
    \rowcolor{ourhighlight}  \methodname{} (Ours)& $832\!\times\!480\!\times\!81$ & 4 & 11s & \textbf{85.50} & \textbf{80.27} & \textbf{84.46} \\
    
    \midrule
    % ------------------------------------------------------------------
    \multicolumn{7}{c}{\textit{Wan2.1-T2V-14B}} \\
    \midrule
    Teacher (Pretrained) & $832\!\times\!480\!\times\!81$ & $50{\times}2$ & 710s & 84.71 & 81.75 & 84.12 \\
    \midrule
    DMD~\cite{yin2024one}            & $832\!\times\!480\!\times\!81$ & 4 & 27s & 84.56 & 80.84 & 83.81 \\
    sCM~\cite{lu2024simplifying}     & $832\!\times\!480\!\times\!81$ & 6 & 44s & 83.48 & 74.01 & 81.52 \\
    rCM~\cite{zheng2025large}        & $832\!\times\!480\!\times\!81$ & 4 & 27s & 84.97 & 80.35 & 84.05 \\
    AVD~\cite{you2026adaptive}       & $832\!\times\!480\!\times\!81$ & 4 & 27s & 84.19 & 82.07 & 83.89 \\
    \rowcolor{ourhighlight}  \methodname{} (Ours)             & $832\!\times\!480\!\times\!81$ & 4 & 27s & \textbf{85.55} & \textbf{82.12} & \textbf{84.87} \\
    \bottomrule
  \end{tabular}
\end{table*}

% \vspace{-0.4cm}

%------------
\section{Experiments}
\label{sec:exp}

\subsection{Experimental setup}
\label{subsec:exp_setup}

\paragraph{Models \& Data.}
We distill the official Wan-2.1-T2V teachers (1.3B and 14B)~\cite{wan2025wan}, both released as 50-step models with classifier-free guidance ($50{\times}2$ NFE), into 4-step students. All clips are generated at $832{\times}480$, 81 frames (5\,s at 16\,fps). The distribution-matching loss uses text descriptions from a mixed corpus of public videos and never accesses the raw videos themselves, following~\cite{yin2024one,yin2024improved}.

\paragraph{Training \& Evaluation.}
We train all models using the AdamW optimizer with $\beta_1=0.9$ and $\beta_2=0.999$, a classifier-free guidance (CFG) scale of 3.5, and BF16 mixed-precision training on 32 GPUs. The relational temperature is set to $\tau=0.1$, and the two relational weights are fixed at $\lambda_{\text{batch}}=\lambda_{\text{frame}}=0.1$. Full training hyperparameters are provided in Table~\ref{tab:training-config} in the Appendix.

For evaluation, we follow the VBench protocol~\cite{huang2023vbench,zheng2025vbench2} and report quality, semantic, and total scores. All baselines are re-evaluated in a unified codebase under BF16 precision for a fair comparison. Specifically, we use the official checkpoints for rCM and DCM. Since the official weights for sCM and AVD are not publicly available, we evaluate our reproduced implementations of these methods based on their released codebases and paper descriptions. DMD serves as our primary strong baseline.

%-----------------------------------------------------------------------

\subsection{Main Results}
\label{subsec:comp_sota}

\paragraph{Quantitative comparison.}
Table~\ref{tab:vbench} reports VBench results against trajectory-based samplers (sCM~\cite{lu2024simplifying}, rCM~\cite{zheng2025large}) and distribution-based methods (DMD~\cite{yin2024one}, DCM~\cite{lv2025dual}, AVD~\cite{you2026adaptive}). Our \methodname{} achieves the best total score at both model scales, reaching \textbf{84.46} on 1.3B and \textbf{84.87} on 14B. Compared with the DMD baseline, the gains are about $+1.1$ points at both scales. The full per-attribute breakdown is provided in Appendix~\ref{app:vbench_detail}, where \methodname{} leads on the majority of individual VBench dimensions.
The improvement is more pronounced on the semantic axis than on the quality axis. This trend matches the diagnosis, that the semantic aspects such as spatial layout, object persistence, and motion plausibility depend more strongly on relational geometry, which is explicitly constrained by our copula-aware supervision.

\begin{figure}[t]
\centering
\includegraphics[width=\linewidth]{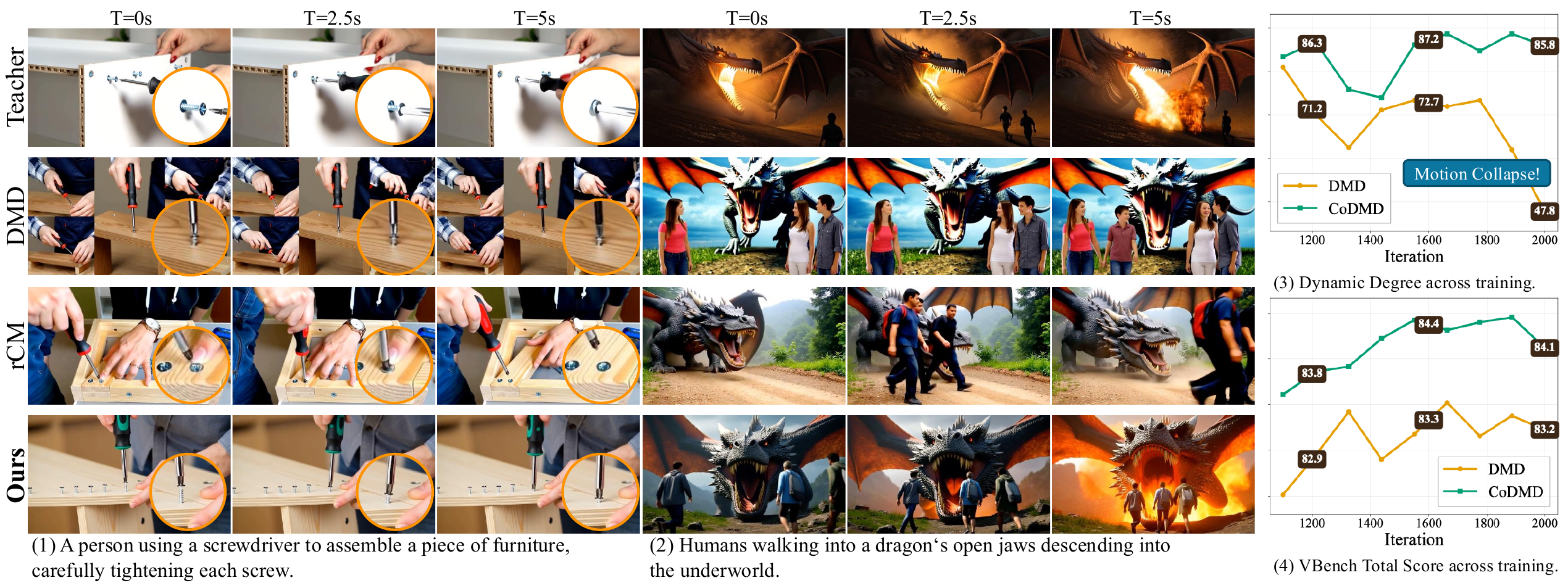}
\caption{\emph{Left:} Our \methodname{} produces more faithful motion, stronger prompt alignment, and more realistic color than previous methods (i.e., DMD and rCM). Owing to the copula-aware relational objective, our output also matches the teacher's visual structure more closely, e.g., (1)screw turns inward; (2) dragon exhales fire.
\emph{Right:} VBench metrics across training iterations. {CoDMD} sustains high scores throughout training, while {DMD} undergoes a sharp collapse in dynamic degree.}
\label{fig:qualitative}
\end{figure}

\textbf{Qualitative comparison.}
We compare videos generated by DMD, rCM, and \methodname{} in Fig.~\ref{fig:teaser} and Fig.~\ref{fig:qualitative}~(left). DMD shows several recurring failure modes, including static camera motion, unnaturally slow dynamics, structural breakdown, and oversaturation. rCM improves some of these issues but introduces its own artifacts, especially broken hand structures and semantic misalignments. In contrast, \methodname{} produces more faithful motion, better semantic alignment, and more realistic color rendering. 
\begin{table}[b]
  \centering
  \scriptsize
  \caption{Ablation study on relational–loss variants.
           We add the batch–level loss $\mathcal{L}_{\text{rel}}^{\text{batch}}$
           and the frame–level loss $\mathcal{L}_{\text{rel}}^{\text{frame}}$
           to the DMD baseline individually and jointly.
           The last row (“Variant”) applies the relational objective
           directly between student and teacher features (no score–difference term).}
  \label{tab:ablation}
  \setlength{\tabcolsep}{8pt}
  \begin{tabular}{l c | c | c c | c}
    \toprule
    \textbf{Method} & \textbf{NFE} &
    \textbf{Dynamic$\uparrow$} & \textbf{Quality$\uparrow$} &
    \textbf{Semantic$\uparrow$}  & \textbf{Total$\uparrow$} \\
    \midrule
    % Teacher (Pretrained)                                            & $50{\times}2$ & --      & --      & --      & --      \\
    DMD                                                       & 4    & 71.1            & 84.5            & 78.8            & 83.4            \\
    DMD + $\mathcal{L}_{\text{rel}}^{\text{batch}}$           & 4    & 83.1            & 85.1            & 79.4            & 84.0            \\
    DMD + $\mathcal{L}_{\text{rel}}^{\text{frame}}$           & 4    & \textbf{93.3}   & 85.0            & 76.8            & 83.9            \\
    \rowcolor{ourhighlight} \textbf{DMD + $\mathcal{L}_{\text{rel}}$  (Ours)}         & 4    & 86.1            & \textbf{85.5}   & \textbf{80.3}   & \textbf{84.5}   \\
    DMD + $\mathcal{L}_{\text{rel}}^{\text{direct}}$ (Variant)& 4    & 88.6            & 84.7            & 79.0            & 83.8            \\
    \bottomrule
  \end{tabular}
\end{table}

\textbf{Copula-aware loss analysis.}
Fig.~\ref{fig:qualitative} further dissects how the two granularities of our relational loss provide complementary regularization. On the left, the batch-level copula loss enforces the student to preserve the relational geometry among video samples defined by the teacher. This explicitly aligns the student's output with the teacher's visual structure: (1) the \emph{screw's inward rotation} and (2) the \emph{dragon's fire-breathing} faithfully mirror the teacher's visual patterns, confirming that sample-level relational matching recovers the copula structure. On the right, DMD undergoes a \emph{sharp mode collapse} in dynamic degree during training, a direct symptom of reverse-KL mode-seeking that drives the generator toward static, low-motion outputs. CoDMD sustains consistently high dynamic degree throughout training, demonstrating that the frame-level relational penalty effectively mitigates motion collapse and preserves the teacher's temporal diversity.

\paragraph{User study.}
\begin{wrapfigure}{r}{0.61\textwidth}
  \centering
  \vspace{-12pt}
  \includegraphics[width=0.61\textwidth]{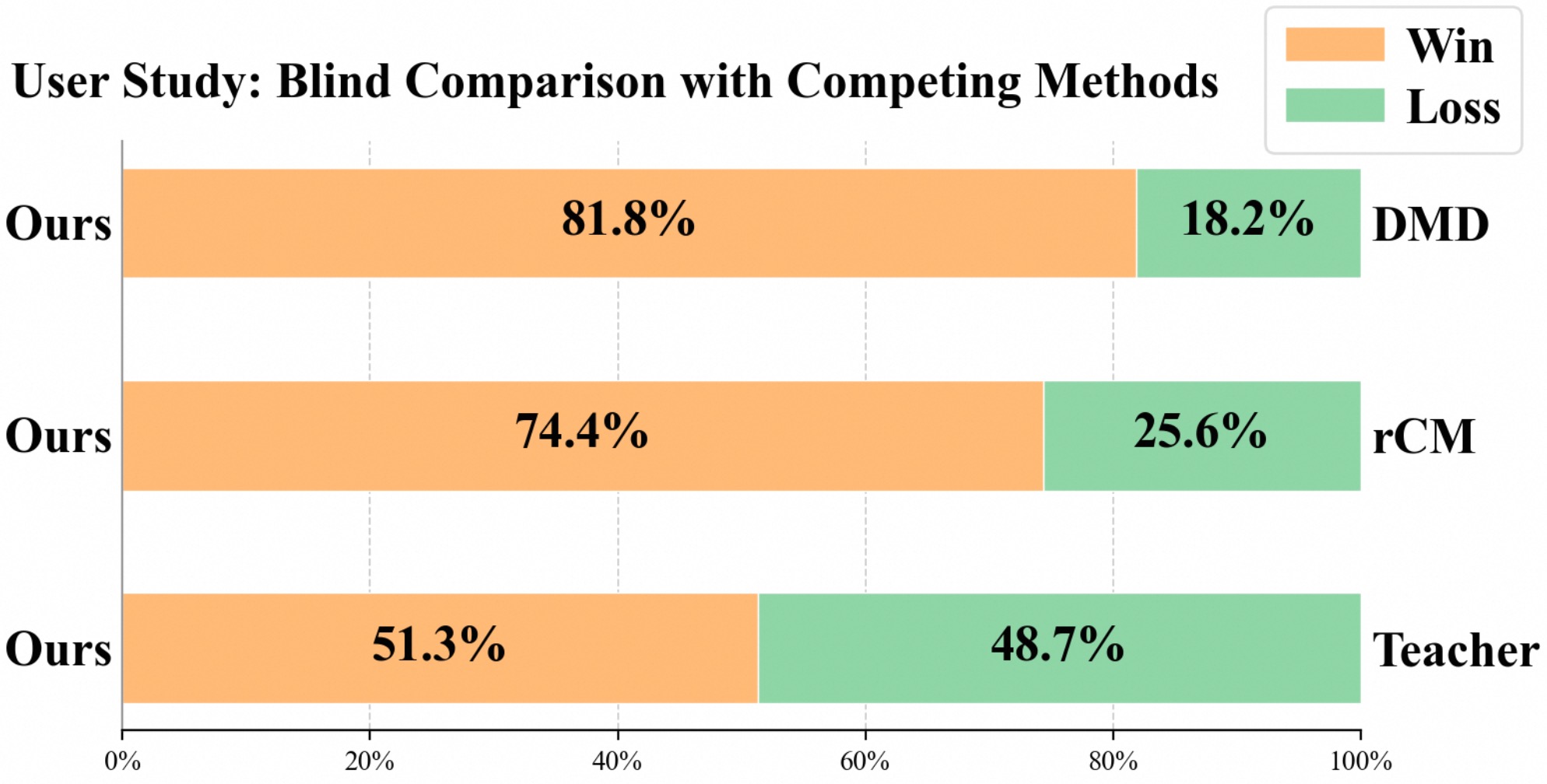}
  \vspace{-6pt}
  \caption{User study results. We report the win rate of \methodname{} against others in pairwise blind comparisons.}
  \label{fig:user_study}
  \vspace{-10pt}
\end{wrapfigure}
To complement the automated metrics, we conduct a blind pairwise user study with 25 evaluators. Each participant is presented with a pair of videos generated from the same prompt in randomised left--right order, and asked to select the one with better \emph{visual quality} and \emph{semantic alignment}. As shown in Fig.~\ref{fig:user_study}, \methodname{} is preferred over DMD in 81.8\% of comparisons and over rCM in 74.4\%, confirming that our distilled model delivers perceptually meaningful improvements over existing methods. Against the 50-step teacher, the win rate reaches 51.3\%, indicating that our methods closely approaches the full-step generation quality. Details are available in Appendix.~\ref{app:user_study_details}.

\subsection{Ablation studies}
\label{subsec:ablation}

\paragraph{The two granularities are complementary.}
Table~\ref{tab:ablation} isolates each component under the same setting. The DMD baseline reaches 83.4 in total score with dynamic degree 71.1. Adding only $\mathcal{L}_\text{rel}^\text{batch}$ improves all metrics and raises the total to 84.0, while adding only $\mathcal{L}_\text{rel}^\text{frame}$ yields the largest motion gain (dynamic degree 93.3) but reduces semantic score by 2.0, indicating that temporal relational matching alone can over-emphasize motion at the expense of prompt alignment. Combining both losses gives the best total of 84.5, with dynamic degree settling at 86.1 and semantic score peaking at 80.3, consistent with the complementary roles analyzed in \S\ref{subsec:copula_theory}.

\paragraph{The score-as-target form is essential.}
Replacing the teacher--fake residual with direct matching between student and teacher score similarity matrices lowers the total score to 83.8, only slightly above DMD. As discussed in Appendix~\ref{app:rel_grad_derivation}, direct relational matching removes the real-versus-fake distribution-matching signal and therefore provides much weaker supervision.

\begin{wraptable}{r}{0.45\textwidth}
  \vspace{-0.6cm}
  \centering
  \caption{Sensitivity to $\lambda $ on (1.3B, 4-step).}
  \label{tab:ablation_weight}
  \scriptsize
  \setlength{\tabcolsep}{7pt}
  \begin{tabular}{c|ccc}
    \toprule
    $\lambda_\text{batch}, \lambda_\text{frame}$ & \textbf{Qual.$\uparrow$} & \textbf{Sem.$\uparrow$} & \textbf{Total$\uparrow$} \\
    \midrule
    0 (DMD)  & 84.53 & 78.78 & 83.38 \\
    \rowcolor{ourhighlight} 0.1      & \textbf{85.50} & 80.27 & \textbf{84.46} \\
    0.2      & 85.04 & \textbf{80.42} & 84.12 \\
    0.5      & 84.79 & 80.40 & 83.91 \\
    1.0      & 83.87 & 78.00 & 82.70 \\
    \bottomrule
  \end{tabular}
  \vspace{-0.3cm}
\end{wraptable}

\paragraph{Sensitivity to relational loss weight.}
Table~\ref{tab:ablation_weight} varies the shared weight
$\lambda = \lambda_\text{batch} = \lambda_\text{frame}$ from 0 (pure DMD)
to 1.0. Any nonzero $\lambda$ improves over the DMD baseline, with
the best total score at $\lambda{=}0.1$ and the best semantic score at
$\lambda{=}0.2$. Performance remains robust across $[0.1, 0.5]$, confirming
that the relational objective provides a stable benefit without sensitive
tuning. At $\lambda{=}1.0$ the relational term overwhelms the
coordinate-wise DMD signal, causing a notable drop in scores. We adopt $\lambda{=}0.1$ as the default throughout
all other experiments.
\section{Conclusion}
\label{sec:conclusion}

We identified that standard DMD's objective leaves the copula structure across batch elements and temporal frames less constrained. Combined with reverse-KL mode-seeking, this gap drives few-step video generators toward layout instability, oversaturation, and motion collapse. We addressed it with \methodfull{}, which lifts DMD's score-as-target surrogate from coordinates to pairwise similarity matrices at batch and frame granularities, reusing score predictions already computed at each step without extra networks, data, or inference overhead.
On Wan-2.1-T2V series, \methodname{} distills 50-step teachers into 4-step students that outperform prior methods on VBench across different attributes. Ablations confirm the two granularities are complementary and the score-as-target formulation is essential. 
We discuss the limitations and future works in Appendix~\ref{sec:limitations}.

% , further supporting our copula-aware view of fast video distillation.

\clearpage
\newpage

\clearpage
{
\small
\bibliographystyle{plainnat}
\bibliography{main}
}

%%%%%%%%%%%%%%%%%%%%%%%%%%%%%%%%%%%%%%%%%%%%%%%%%%%%%%%%%%%%

\clearpage
\newpage

\appendix
%%%%%%%%%%%%%%%%%%%%%%%%%%%%%%%%%%%%%%%%%%%%%%%%%%%%%%%%%%%%%%%%%%%%%%%%
% Appendix file for CoDMD.
% Include in main.tex via:  \appendix \input{appendix}
%%%%%%%%%%%%%%%%%%%%%%%%%%%%%%%%%%%%%%%%%%%%%%%%%%%%%%%%%%%%%%%%%%%%%%%%

\newtheorem{proposition}{Proposition}

% \section*{\ourc{Appendix for }\\
% CoDMD: Copula-aware Distribution Matching Distillation}

% \thispagestyle{plain}
\begin{center}
    {\LARGE \textbf{Appendix} }
    \vspace{0.1cm}
    {\large for Copula-aware Distribution Matching Distillation }
    % \vspace{0.8cm}
\end{center}

% \noindent
% This appendix provides supplementary material organized as follows:

% \vspace{0.3cm}
\noindent
\begin{tabular}{@{}p{0.10\linewidth} p{0.80\linewidth}@{}}
    \hyperlink{app:visual}{\S A} & \textbf{More Visual Comparisons} \\[0.25cm]
    \hyperlink{app:hyperparams}{\S B} & \textbf{Training Hyperparameters} (Table~\ref{tab:training-config}) \\[0.25cm]
    \hyperlink{app:vbench_detail}{\S C} & \textbf{Full VBench Attribute Breakdown} (Table~\ref{tab:vbench_full})  \\[0.25cm]
    \hyperlink{app:User}{\S D} & \textbf{User Study Details} \\[0.25cm]
    \hyperlink{app:hyperparams}{\S E} & \textbf{Compute Resources} \\[0.25cm]
    \hyperlink{app:theory}{\S F} & \textbf{Theoretical Analysis} \\[0.08cm]
    & \quad \hyperlink{app:theory:c1}{F.1}\;\; Marginal matching does not identify relational structure \\[0.05cm]
    & \quad \hyperlink{app:theory:c2}{F.2}\;\; Reverse-KL selects dominant modes under restricted families \\[0.05cm]
    & \quad \hyperlink{app:theory:c3}{F.3}\;\; Relational matching penalizes collapsed solutions \\[0.05cm]
    & \quad \hyperlink{app:theory:c4}{F.4}\;\; Gradient derivation of the relational loss \\[0.25cm]
    \hyperlink{app:limit}{\S G} & \textbf{Limitation and Future work}  \\[0.25cm]
    \hyperlink{app:notation}{\S H} & \textbf{Notation Table} (Table~\ref{tab:notation})
\end{tabular}

% \vspace{0.5cm}
% \noindent
\rule{\linewidth}{0.5pt}
\vspace{-0.5cm}

\hypertarget{app:visual}{}
\section{More visual comparisons.}
% We also provide \ourc{demo videos in the Supplementary Material} together with a local webpage that enables convenient comparison, and we encourage reviewers to examine them.

\begin{figure*}[h!]
  \centering
    \includegraphics[width=\linewidth]{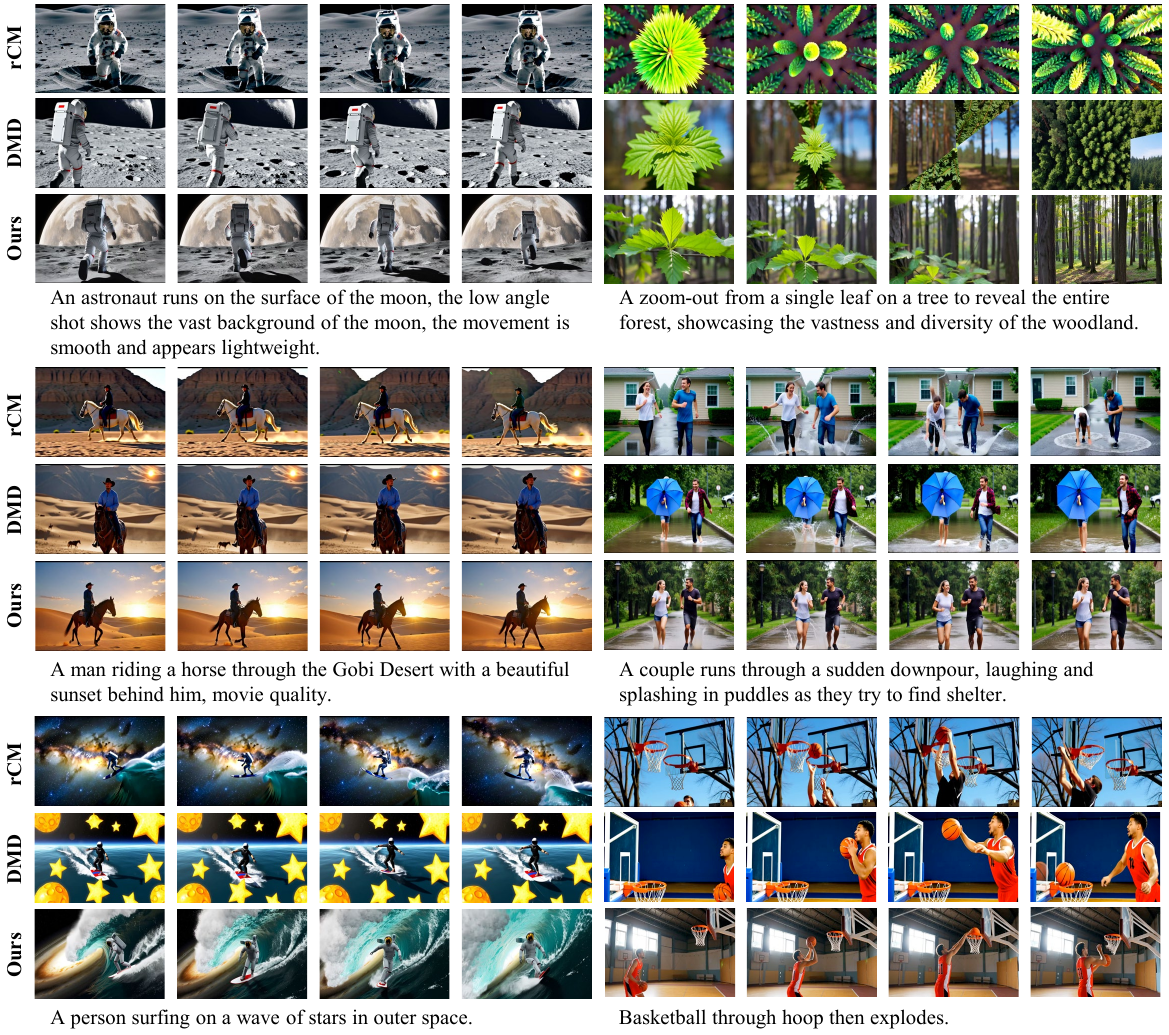}
\end{figure*}

\begin{figure*}[h!]
  \centering
    \includegraphics[width=\linewidth]{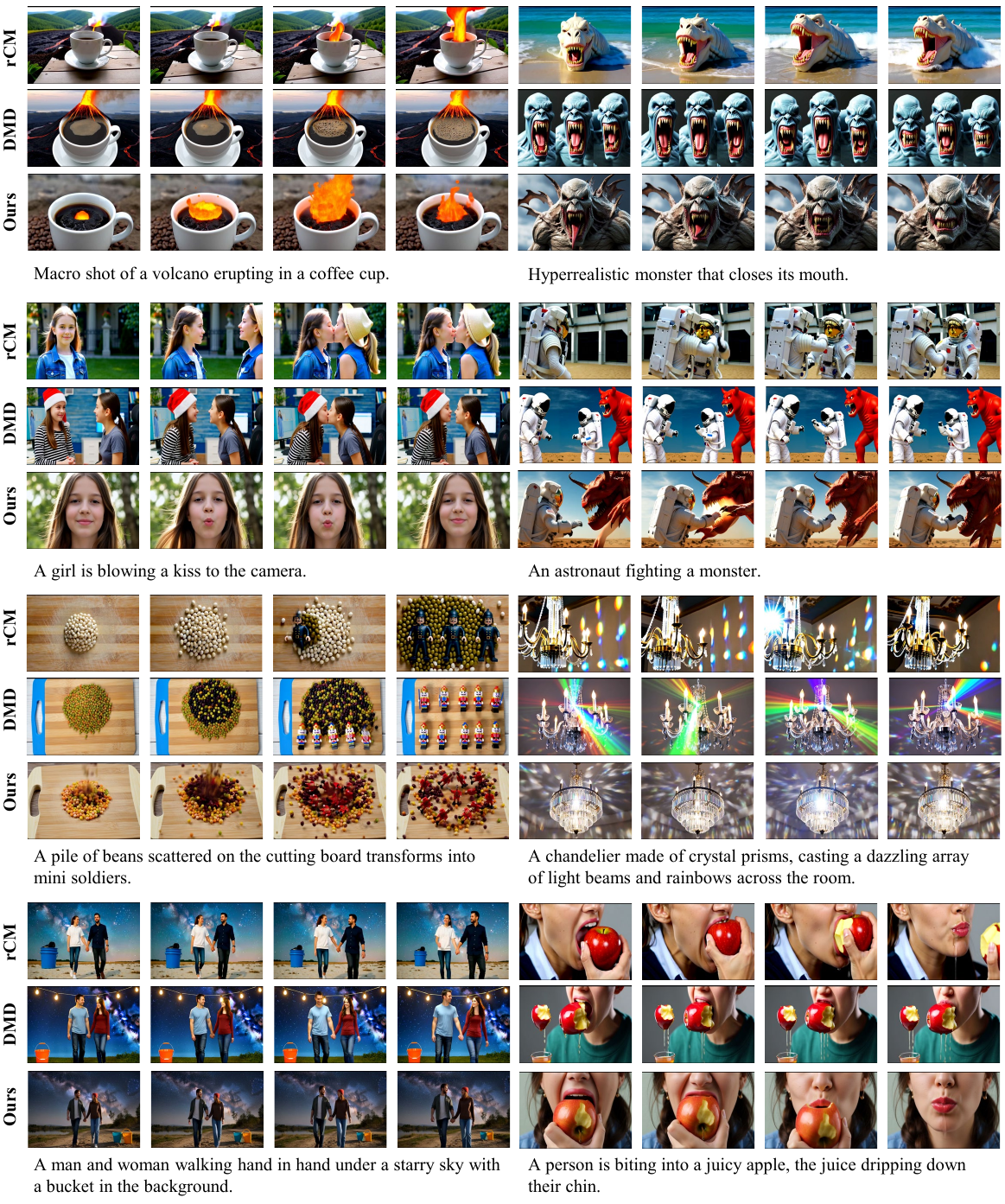}
\end{figure*}

\hypertarget{app:hyperparams}{}
\section{Training Hyperparameters}
\label{sec:train_config}

\begin{table}[h!]
  \centering
  \caption{Training configurations for Wan2.1 T2V.}
  \label{tab:training-config}
  \begin{tabular}{lcc}
    \toprule
    \textbf{Hyper-parameter} & \textbf{1.3B} & \textbf{14B} \\
    \midrule
    % EMA Length                     & 0.05  & 0.05  \\
    Batch Size                     & 128   & 64    \\
    % Context Parallel Size          & 1     & 10    \\
    Learning Rate (student)        & 2e-6  & 2e-6  \\
    Learning Rate (fake score)     & 4e-7  & 4e-7  \\
    CFG Scale                      & 3.5   & 3.5   \\
    Student Update Frequency       & 5     & 5    \\
    Maximal Simulation Steps       & 4     & 4     \\
    % Tangent Warm-up Iterations     & 1000  & 200   \\
    % Total Iterations               & 5k   & 5k   \\
    % $\sigma_{\max}$                & 1600  & 1600  \\
    \bottomrule
  \end{tabular}
\end{table}

\hypertarget{app:vbench_detail}{}
\section{Full VBench Attribute Breakdown}
\label{app:vbench_detail}

Table~\ref{tab:vbench_full} reports the per-attribute VBench scores for
all compared methods at both the 1.3B and 14B scales.
\methodname{} achieves the highest total score at both scales and leads
on the majority of individual attributes, with the most pronounced
gains on \emph{dynamic degree}, \emph{color}, \emph{multiple objects},
and \emph{spatial relationship}. The teacher retains advantages on a few
stability-oriented metrics such as \emph{background consistency} and
\emph{temporal flickering}, which reflects the inherent trade-off
between motion richness and frame-level steadiness under reduced NFE.

\begin{table*}[t]
  \centering
  \caption{Full VBench attribute-level results on Wan-2.1-T2V at 1.3B and 14B scales. All scores are reported as percentages. \colorbox{ourhighlight}{Orange} columns denote our \methodname{}.}
  \label{tab:vbench_full}
  \setlength{\tabcolsep}{5pt}
  \renewcommand{\arraystretch}{1.05}
  \definecolor{ourhighlight}{RGB}{255,237,215}
  \begin{tabular}{l|cccc|cccc}
    \toprule
    & \multicolumn{4}{c|}{\textit{Wan2.1-T2V-1.3B}}
    & \multicolumn{4}{c}{\textit{Wan2.1-T2V-14B}} \\
    \textbf{Attribute}
    & Teacher & rCM & DMD & \cellcolor{ourhighlight}\textbf{Ours}
    & Teacher & rCM & DMD & \cellcolor{ourhighlight}\textbf{Ours} \\
    \midrule
    Subject Consistency      & 93.46 & 94.11 & 95.18 & \cellcolor{ourhighlight}\textbf{96.08} & 93.58 & 94.87 & 96.14 & \cellcolor{ourhighlight}95.92 \\
    Background Consistency   & 96.29 & 93.88 & 94.65 & \cellcolor{ourhighlight}94.23 & \textbf{96.76} & 94.44 & 95.44 & \cellcolor{ourhighlight}95.03 \\
    Temporal Flickering      & 99.05 & 97.53 & \textbf{99.30} & \cellcolor{ourhighlight}98.26 & 98.90 & 97.66 & \textbf{99.21} & \cellcolor{ourhighlight}98.37 \\
    Motion Smoothness        & 97.97 & 97.75 & 98.50 & \cellcolor{ourhighlight}\textbf{98.67} & 98.02 & 97.83 & 98.54 & \cellcolor{ourhighlight}\textbf{98.51} \\
    Dynamic Degree           & 67.22 & 83.06 & 71.11 & \cellcolor{ourhighlight}\textbf{86.11} & 68.61 & 85.28 & 64.44 & \cellcolor{ourhighlight}\textbf{81.39} \\
    Aesthetic Quality        & \textbf{68.94} & 64.60 & 64.89 & \cellcolor{ourhighlight}65.53 & \textbf{69.28} & 66.78 & 66.36 & \cellcolor{ourhighlight}66.94 \\
    Imaging Quality          & 67.58 & 66.80 & \textbf{68.05} & \cellcolor{ourhighlight}67.99 & 67.82 & \textbf{69.37} & 68.04 & \cellcolor{ourhighlight}68.62 \\
    \midrule
    Object Class             & 91.01 & 93.86 & 94.05 & \cellcolor{ourhighlight}\textbf{95.11} & 92.63 & 96.39 & 96.58 & \cellcolor{ourhighlight}\textbf{96.28} \\
    Multiple Objects         & 78.17 & 78.35 & 78.72 & \cellcolor{ourhighlight}\textbf{81.88} & 83.58 & 82.50 & 84.36 & \cellcolor{ourhighlight}\textbf{86.08} \\
    Human Action             & 96.80 & 97.20 & 96.80 & \cellcolor{ourhighlight}\textbf{97.60} & 98.40 & 99.20 & \textbf{99.40} & \cellcolor{ourhighlight}\textbf{99.40} \\
    Color                    & 85.00 & 82.82 & 83.76 & \cellcolor{ourhighlight}\textbf{86.52} & 85.35 & 82.63 & 80.37 & \cellcolor{ourhighlight}\textbf{85.19} \\
    Spatial Relationship     & 81.12 & 76.86 & 76.24 & \cellcolor{ourhighlight}\textbf{78.41} & 81.96 & 78.19 & 79.30 & \cellcolor{ourhighlight}\textbf{84.00} \\
    Scene                    & \textbf{56.56} & 54.56 & 50.86 & \cellcolor{ourhighlight}54.20 & 53.82 & 53.66 & 53.75 & \cellcolor{ourhighlight}\textbf{54.62} \\
    Appearance Style         & \textbf{22.88} & 21.49 & 21.60 & \cellcolor{ourhighlight}21.25 & \textbf{23.59} & 21.72 & 22.49 & \cellcolor{ourhighlight}21.71 \\
    Temporal Style           & \textbf{25.90} & 25.08 & 25.05 & \cellcolor{ourhighlight}25.35 & \textbf{25.73} & 25.19 & 25.23 & \cellcolor{ourhighlight}25.92 \\
    Overall Consistency      & \textbf{27.09} & 26.60 & 26.63 & \cellcolor{ourhighlight}26.57 & \textbf{27.34} & 26.85 & 26.99 & \cellcolor{ourhighlight}27.14 \\
    \midrule
    \rowcolor{gray!8}
    Quality Score            & 84.43 & 83.72 & 84.53 & \cellcolor{ourhighlight}\textbf{85.50} & 84.71 & 84.97 & 84.56 & \cellcolor{ourhighlight}\textbf{85.55} \\
    \rowcolor{gray!8}
    Semantic Score           & 80.73 & 79.18 & 78.78 & \cellcolor{ourhighlight}\textbf{80.27} & 81.75 & 80.35 & 80.84 & \cellcolor{ourhighlight}\textbf{82.12} \\
    \rowcolor{gray!8}
    Total Score              & 83.69 & 82.81 & 83.38 & \cellcolor{ourhighlight}\textbf{84.46} & 84.12 & 84.05 & 83.81 & \cellcolor{ourhighlight}\textbf{84.87} \\
    \bottomrule
  \end{tabular}
\end{table*}

\hypertarget{app:User}{}
\section{User Study Details}
\hypertarget{app:hyperparams}{}
\label{app:user_study_details}

To complement the automated metrics, we conduct a blind pairwise user study to compare \methodname{} with DMD, rCM, and the 50-step teacher. The study involves 25 evaluators. In each trial, a participant is shown two videos generated from the same text prompt, with the left--right display order randomized. The source methods are hidden from the participant. The participant is asked to judge which video is better with respect to visual quality and semantic alignment to the prompt.

We use the following instruction for all participants:

\begin{quote}
You will be shown two videos generated from the same text prompt. Please compare the two videos and select the one that is better overall, considering both visual quality and semantic alignment to the prompt. The presentation order is randomized, and the source methods are not disclosed. Please make your judgment only based on the displayed videos and the prompt.
\end{quote}

The evaluation is conducted in a blind manner. Each participant reviews video pairs from multiple method comparisons, including \methodname{} versus DMD, \methodname{} versus rCM, and \methodname{} versus the teacher model. We record the pairwise preference rate of \methodname{} over each baseline. The aggregated results are reported in Fig.~\ref{fig:user_study}.

This study is minimal risk. Participants are only asked to compare generated videos and no sensitive personal information is collected. Participation is voluntary, and participants may stop at any time.

\paragraph{Compensation.}
% [Participants were not financially compensated.]
% or
All participants received compensation according to local institutional practice.

\paragraph{Ethics review.}
No formal IRB or equivalent approval was obtained for this minimal-risk study.
% or
% [This study was reviewed through the relevant institutional ethics review process.]

\paragraph{Interface.}
The user study is conducted with a local side-by-side comparison interface that displays the prompt together with two generated videos. The interface supports both desktop and mobile layouts, with A/B video panels adapted to different screen sizes for convenient access across devices. An anonymized screenshot of the interface is shown in Fig.~\ref{fig:user_study_interface}.

\begin{figure}[h]
    \centering
    \includegraphics[width=\linewidth]{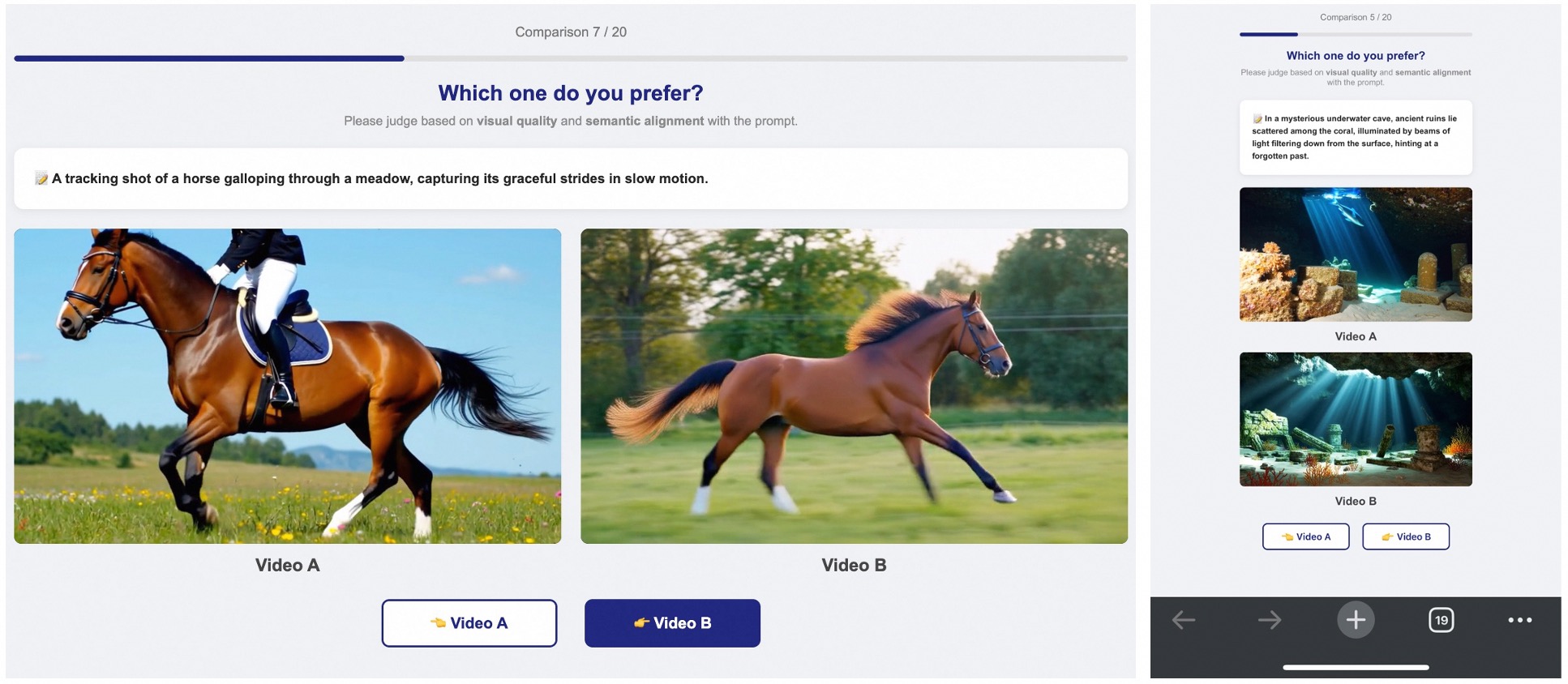}
    \caption{Anonymized interface used in the blind pairwise user study. The interface displays the prompt together with two generated videos in randomized left--right order and supports both desktop and mobile layouts.}
    \label{fig:user_study_interface}
\end{figure}

\section{Compute Resources}
\label{app:compute}
Each main training in the manuscript is conducted on 32 NVIDIA A100 GPUs with 80GB memory and takes approximately 1.5 days. We report the compute configuration for the experiments in the paper, while additional exploratory runs may require extra compute beyond the final reported settings.

\hypertarget{app:theory}{}
\section{Theoretical Analysis}
\label{sec:theory}

This appendix supplies the formal statements and proofs of the three
propositions introduced in \S\ref{subsec:copula_theory} of the main
paper. These results identify the structural mechanism behind DMD's
fragility in the few-step video regime and motivate the design choices
in \methodname{}.

Let $X=(X_1,\dots,X_F)$ denote a video with $F$ frames, $P$ the
teacher-induced distribution and $Q$ the student-induced distribution,
with $P_t,Q_t$ the frame-wise marginals. For a symmetric distance
$d(\cdot,\cdot)$ on score-based features, the pairwise relational
statistic
\[
    \mathcal{D}(Q)= \mathbb{E}_{Z,Z'\sim Q}\bigl[d(Z,Z')\bigr]
\]
captures the relational geometry of a distribution. This statistic is
evaluated at two granularities in \methodname{}: across batch elements
and across temporal frames (\S\ref{subsec:codmd}).

\hypertarget{app:theory:c1}{}
\subsection{Marginal Matching Does Not Identify Relational Structure}

We first formalize the limitation of coordinate-wise or marginal
supervision.

\begin{proposition}[Relational non-identifiability under marginal matching]
\label{prop:marginal_nonidentifiability}
There exist two video distributions $P$ and $Q$ such that
\[
    P_t = Q_t, \qquad \forall t\in\{1,\dots,T\},
\]
but their temporal relational structures differ. Consequently, any
training objective that depends only on the frame-wise marginals
$\{Q_t\}_{t=1}^T$ cannot distinguish $Q$ from $P$.
\end{proposition}

\begin{proof}
It suffices to construct a two-frame example. Let $U\sim \mathcal{N}(0,1)$
and define
\[
    P: \quad (X_1,X_2)=(U,U),
\]
while
\[
    Q: \quad (X_1,X_2)=(U,-U).
\]
For both distributions, each individual frame marginal is standard
Gaussian:
\[
    P_1=P_2=Q_1=Q_2=\mathcal{N}(0,1).
\]
Hence any objective depending only on the frame-wise marginals assigns
the same value to $P$ and $Q$.

However, their temporal relations are different. For example,
\[
    \mathbb{E}_{P}[X_1X_2]=1,
    \qquad
    \mathbb{E}_{Q}[X_1X_2]=-1.
\]
Equivalently, the expected squared frame difference satisfies
\[
    \mathbb{E}_{P}\bigl[(X_1-X_2)^2\bigr]=0,
    \qquad
    \mathbb{E}_{Q}\bigl[(X_1-X_2)^2\bigr]=4.
\]
Thus the two distributions have identical frame-wise marginals but
opposite temporal dependence. This proves that marginal matching alone
does not identify relational structure.
\end{proof}

Motion, layout consistency and semantic coherence depend on
couplings among frames and among generated samples; these couplings
can change even when all individual marginals remain matched.

\hypertarget{app:theory:c2}{}
\subsection{Reverse-KL Selects Dominant Modes Under Restricted Families}

We next formalize the mode-seeking behavior of reverse-KL in a restricted
family. Unlike an informal statement about ``mode collapse'', the following
result gives an exact calculation.

\begin{proposition}[Reverse-KL selection in a disjoint mixture]
\label{prop:reversekl_mixture}
Let the teacher distribution be a mixture
\[
    P = \sum_{k=1}^{K} \pi_k P_k,
    \qquad
    \pi_k>0,\quad \sum_{k=1}^{K}\pi_k=1,
\]
where the component distributions $P_k$ have disjoint supports
$A_k$. Suppose the student family is restricted to single-component
approximants
\[
    \mathcal{Q}_{\mathrm{single}}
    =
    \{P_1,\dots,P_K\}.
\]
Then
\[
    D_{\mathrm{KL}}(P_j\,\|\,P)
    =
    -\log \pi_j.
\]
Therefore,
\[
    \arg\min_{Q\in\mathcal{Q}_{\mathrm{single}}}
    D_{\mathrm{KL}}(Q\,\|\,P)
    =
    P_{j^\star},
    \qquad
    j^\star\in\arg\max_j \pi_j.
\]
That is, reverse-KL selects a dominant mixture component rather than
covering all teacher-supported components.
\end{proposition}

\begin{proof}
Since the supports are disjoint, for $x\in A_j$ we have
\[
    p(x)=\pi_j p_j(x).
\]
Therefore,
\[
\begin{aligned}
    D_{\mathrm{KL}}(P_j\,\|\,P)
    &=
    \int_{A_j}
    p_j(x)
    \log
    \frac{p_j(x)}{p(x)}
    \,dx  \\
    &=
    \int_{A_j}
    p_j(x)
    \log
    \frac{p_j(x)}{\pi_j p_j(x)}
    \,dx  \\
    &=
    \int_{A_j}
    p_j(x)
    \log
    \frac{1}{\pi_j}
    \,dx  \\
    &=
    -\log \pi_j.
\end{aligned}
\]
Thus minimizing $D_{\mathrm{KL}}(P_j\|P)$ is equivalent to maximizing
$\pi_j$, proving the claim.
\end{proof}

This result captures the precise sense in which reverse-KL is
mode-seeking under a restricted student family. A more general
version holds: if $Q_j$ is supported on $A_j$, then
$D_{\mathrm{KL}}(Q_j\|P)= D_{\mathrm{KL}}(Q_j\|P_j)-\log\pi_j$.
Reverse-KL therefore trades off per-component fit against mixture
weight, but never penalizes the absence of uncovered modes.

A slightly more general version is also useful. If a candidate $Q_j$ is
supported on $A_j$, then
\[
    D_{\mathrm{KL}}(Q_j\,\|\,P)
    =
    D_{\mathrm{KL}}(Q_j\,\|\,P_j)
    -
    \log \pi_j.
\]
Thus reverse-KL trades off how well $Q_j$ fits component $P_j$ against the
mixture weight $\pi_j$, but still does not directly penalize the absence
of other components.

\hypertarget{app:theory:c3}{}
\subsection{Relational Matching Penalizes Collapsed Solutions}

We now show that pairwise relational matching provides a mechanism that
standard reverse-KL lacks: it can penalize a student that collapses to one
component of a multi-relation teacher distribution.

Let $d(z,z')$ be a nonnegative symmetric distance in feature space, and
define
\[
    \mathcal{D}(Q)
    =
    \mathbb{E}_{Z,Z'\sim Q}
    \bigl[
        d(Z,Z')
    \bigr].
\]
For the teacher mixture
$P=\sum_{k=1}^{K}\pi_k P_k$, we have
\[
\begin{aligned}
    \mathcal{D}(P)
    &=
    \sum_{a=1}^{K}\sum_{b=1}^{K}
    \pi_a\pi_b
    \mathbb{E}_{Z\sim P_a,\,Z'\sim P_b}
    \bigl[
        d(Z,Z')
    \bigr].
\end{aligned}
\]
The terms with $a\neq b$ encode cross-mode relational diversity.

\begin{proposition}[Relational penalty for single-mode collapse]
\label{prop:relational_penalty}
Assume the teacher distribution is a mixture
$P=\sum_{k=1}^{K}\pi_k P_k$ with disjoint relational modes. Suppose that
for some constants $0\leq \epsilon < \Delta$,
\[
    \mathbb{E}_{Z,Z'\sim P_k}
    \bigl[
        d(Z,Z')
    \bigr]
    \leq \epsilon,
    \qquad \forall k,
\]
while for all $a\neq b$,
\[
    \mathbb{E}_{Z\sim P_a,\,Z'\sim P_b}
    \bigl[
        d(Z,Z')
    \bigr]
    \geq \Delta.
\]
Then any single-mode student $Q=P_j$ satisfies
\[
    \bigl|
        \mathcal{D}(Q)-\mathcal{D}(P)
    \bigr|
    \geq
    \Delta
    \Bigl(
        1-\sum_{k=1}^{K}\pi_k^2
    \Bigr)
    -
    \epsilon.
\]
In particular, when the teacher has multiple well-separated relational
modes, a single-mode student incurs a nonzero relational discrepancy.
\end{proposition}

\begin{proof}
For the teacher distribution,
\[
\begin{aligned}
    \mathcal{D}(P)
    &=
    \sum_{k=1}^{K}
    \pi_k^2
    \mathbb{E}_{Z,Z'\sim P_k}
    [d(Z,Z')]
    +
    \sum_{a\neq b}
    \pi_a\pi_b
    \mathbb{E}_{Z\sim P_a,Z'\sim P_b}
    [d(Z,Z')]  \\
    &\geq
    \sum_{a\neq b}
    \pi_a\pi_b \Delta  \\
    &=
    \Delta
    \Bigl(
        1-\sum_{k=1}^{K}\pi_k^2
    \Bigr).
\end{aligned}
\]
For a single-mode student $Q=P_j$, by assumption,
\[
    \mathcal{D}(Q)
    =
    \mathbb{E}_{Z,Z'\sim P_j}
    [d(Z,Z')]
    \leq \epsilon.
\]
Therefore,
\[
    \mathcal{D}(P)-\mathcal{D}(Q)
    \geq
    \Delta
    \Bigl(
        1-\sum_{k=1}^{K}\pi_k^2
    \Bigr)
    -
    \epsilon,
\]
which proves the claim.
\end{proof}

This result formalizes why pairwise relational matching complements
DMD. Reverse-KL may select a dominant component under the restricted
few-step student family; a relational objective measures the loss of
cross-mode geometry through pairwise distances or similarities.

\subsection{Implication for \methodname}

The three propositions jointly justify the design decisions in
\S\ref{subsec:codmd}. Proposition~\ref{prop:marginal_nonidentifiability}
shows that coordinate-wise supervision cannot identify temporal or
batch-level relational geometry. Proposition~\ref{prop:reversekl_mixture}
shows that reverse-KL selects a dominant mode under the restricted
few-step student family. Proposition~\ref{prop:relational_penalty}
shows that pairwise relational matching assigns a nonzero cost to
single-mode concentration whenever the teacher contains multiple
relational modes. \methodname{} therefore instantiates this relational
penalty via the score-as-target surrogate of
Eq.~\eqref{eq:rel_loss_general}, applied at two complementary
granularities (batch and frame) that together regularize the degrees
of freedom DMD leaves unconstrained. The gradient structure of this
construction is formalized in Appendix~\ref{app:rel_grad_derivation}.

\hypertarget{app:theory:c4}{}
\subsection{Gradient Derivation of the Relationship-Distillation Loss}
\label{app:rel_grad_derivation}

This appendix proves the claim made in \S\ref{subsec:codmd}: the
gradient of the relationship-distillation loss $\mathcal L_\text{rel}$
with respect to the student similarity matrix $\mathbf S_\text{stu}$
is, to first order in $\tau^{-1}$, proportional to the relational
score difference
$\boldsymbol{\Delta}_{\!\mathbf S}=\mathbf S_\text{fake}-\mathbf S_\text{real}$.
This makes the analogy with the DMD gradient~\eqref{eq:dmd_grad}
mathematically precise: \methodname{} and DMD are two instances of the
\emph{same} score-as-target construction, applied at different
granularities (coordinates vs.\ pairwise interactions).

%-----------------------------------------------------------------------
Let $\mathbf S_\text{stu},\mathbf S_\text{real},\mathbf S_\text{fake}\in\mathbb R^{N\times N}$
be the cosine similarity matrices of \S\ref{subsec:codmd}, and
$\boldsymbol{\Delta}_{\!\mathbf S}=\mathbf S_\text{fake}-\mathbf S_\text{real}$.
For brevity, write
\[
\mathbf A \;=\; \tau^{-1}\mathbf S_\text{stu},\qquad
\mathbf B \;=\; \tau^{-1}\bigl(\mathbf S_\text{stu}-\boldsymbol{\Delta}_{\!\mathbf S}\bigr)
\;=\; \mathbf A - \tau^{-1}\boldsymbol{\Delta}_{\!\mathbf S}.
\]
Both matrices are passed through a row-wise softmax, and the gradient
through the target $\mathbf P_\text{tgt}$ is detached:
\[
\mathbf P_\text{stu}^{i,:}\;=\;\sigma(\mathbf A^{i,:}),\qquad
\mathbf P_\text{tgt}^{i,:}\;=\;\mathrm{stopgrad}\!\bigl(\sigma(\mathbf B^{i,:})\bigr),
\]
where $\sigma(\mathbf v)_j = e^{v_j}/\sum_k e^{v_k}$ denotes the softmax
along a row. The relationship-distillation loss
(Eq.~\eqref{eq:rel_loss_general}) is the row-averaged KL,
\begin{equation}\label{eq:rel_loss_appendix}
\mathcal L_\text{rel}
\;=\;
\frac{1}{N}\sum_{i=1}^{N}
\mathrm{KL}\!\bigl(\mathbf P_\text{tgt}^{i,:}\,\big\|\,\mathbf P_\text{stu}^{i,:}\bigr)
\;=\;
\frac{1}{N}\sum_{i=1}^{N}\sum_{j=1}^{N}
P^{i}_{\text{tgt},j}\,\bigl[\log P^{i}_{\text{tgt},j}-\log P^{i}_{\text{stu},j}\bigr].
\end{equation}

We will compute $\partial\mathcal L_\text{rel}/\partial\mathbf S_\text{stu}$
in two steps: (1) the exact closed-form gradient, then (2) its first-order
expansion in the small parameter $\tau^{-1}\boldsymbol{\Delta}_{\!\mathbf S}$.

%-----------------------------------------------------------------------
\subsection*{Step 1: Exact gradient via the softmax Jacobian}

Because $\mathbf P_\text{tgt}$ is detached, only $\log P^{i}_{\text{stu},j}$
in Eq.~\eqref{eq:rel_loss_appendix} carries gradient. Fix a row $i$
and recall the standard log-softmax identity
\[
\log P^{i}_{\text{stu},j}
\;=\;
A_{ij} \;-\; \log\!\sum_{k} e^{A_{ik}}.
\]
Differentiating with respect to $A_{im}$ yields
\begin{equation}\label{eq:logsoftmax_jac}
\frac{\partial \log P^{i}_{\text{stu},j}}{\partial A_{im}}
\;=\;
\mathbf 1[j\!=\!m] \;-\; P^{i}_{\text{stu},m}.
\end{equation}
Substituting Eq.~\eqref{eq:logsoftmax_jac} into the row-$i$ contribution
of Eq.~\eqref{eq:rel_loss_appendix} (and using
$\sum_j P^{i}_{\text{tgt},j}=1$),
\begin{align}
\frac{\partial}{\partial A_{im}}\!\Bigl[\!-\!\sum_{j} P^{i}_{\text{tgt},j}\log P^{i}_{\text{stu},j}\Bigr]
&=
-\sum_{j} P^{i}_{\text{tgt},j}\bigl(\mathbf 1[j\!=\!m]-P^{i}_{\text{stu},m}\bigr) \nonumber\\
&=
P^{i}_{\text{stu},m}-P^{i}_{\text{tgt},m}.
\label{eq:exact_grad_in_A}
\end{align}
Chaining through $A_{im}=\tau^{-1}(S_\text{stu})_{im}$,
\begin{equation}\label{eq:exact_grad_in_S}
\frac{\partial\mathcal L_\text{rel}}{\partial(S_\text{stu})_{im}}
\;=\;
\frac{1}{N\tau}\,\bigl(P^{i}_{\text{stu},m}-P^{i}_{\text{tgt},m}\bigr).
\end{equation}
This is exact for any $\tau>0$. The student is updated only where its
softmax-normalised relational profile differs from the (detached)
target profile, mirroring DMD's principle of ``update only where
critic and teacher disagree.''

%-----------------------------------------------------------------------
\subsection*{Step 2: First-order expansion in $\tau^{-1}$}

To make the link with the score difference
$\boldsymbol{\Delta}_{\!\mathbf S}$ explicit, we expand
$\mathbf P_\text{tgt}-\mathbf P_\text{stu}$ in powers of
$\boldsymbol{\delta}^{i}=\tau^{-1}(\boldsymbol{\Delta}_{\!\mathbf S})^{i,:}$.
The expansion is informative when the per-row entries of
$\boldsymbol{\delta}^{i}$ are moderate; we verify this empirically
in the ``When does the linearisation hold'' paragraph below.

\paragraph{Softmax differential.}
The Jacobian of the softmax at a point $\mathbf v\!\in\!\mathbb R^{N}$ is the
well-known matrix
\begin{equation}\label{eq:softmax_jacobian}
J(\mathbf v) \;=\; \mathrm{diag}(\sigma(\mathbf v)) - \sigma(\mathbf v)\,\sigma(\mathbf v)^{\top}.
\end{equation}
Setting $\mathbf v=\mathbf A^{i,:}$ and noting that
$\mathbf B^{i,:}=\mathbf A^{i,:}-\boldsymbol{\delta}^{i}$, a first-order
Taylor expansion gives
\begin{equation}\label{eq:taylor}
\sigma(\mathbf B^{i,:})
\;=\;
\sigma(\mathbf A^{i,:}) \;-\; J(\mathbf A^{i,:})\,\boldsymbol{\delta}^{i}
\;+\; O\!\bigl(\|\boldsymbol{\delta}^{i}\|^{2}\bigr).
\end{equation}
Hence
\begin{equation}\label{eq:diff_pp}
\mathbf P_\text{tgt}^{i,:}-\mathbf P_\text{stu}^{i,:}
\;=\;
-J(\mathbf A^{i,:})\,\boldsymbol{\delta}^{i}
\;+\; O\!\bigl(\|\boldsymbol{\delta}^{i}\|^{2}\bigr).
\end{equation}

\paragraph{Substituting into the exact gradient.}
Plugging Eq.~\eqref{eq:diff_pp} into Eq.~\eqref{eq:exact_grad_in_S},
\begin{align}
\frac{\partial\mathcal L_\text{rel}}{\partial(S_\text{stu})_{im}}
&=
\frac{1}{N\tau}\,\bigl(P^{i}_{\text{stu},m}-P^{i}_{\text{tgt},m}\bigr)
\nonumber\\
&=
\frac{1}{N\tau}\,\bigl[J(\mathbf A^{i,:})\,\boldsymbol{\delta}^{i}\bigr]_{m}
\;+\; O(\tau^{-3})
\nonumber\\
&=
\frac{1}{N\tau^{2}}\,\bigl[J(\mathbf A^{i,:})\,(\boldsymbol{\Delta}_{\!\mathbf S})^{i,:}\bigr]_{m}
\;+\; O(\tau^{-3}),
\label{eq:linearised_grad_S}
\end{align}
where we used $\boldsymbol{\delta}^{i}=\tau^{-1}(\boldsymbol{\Delta}_{\!\mathbf S})^{i,:}$.
In matrix form,
\begin{equation}\label{eq:linearised_grad_matrix}
\nabla_{\!\mathbf S_\text{stu}}\mathcal L_\text{rel}
\;=\;
\frac{1}{N\tau^{2}}\,
\mathbf M(\mathbf S_\text{stu};\tau)\;\boldsymbol{\Delta}_{\!\mathbf S}
\;+\; O(\tau^{-3}),
\end{equation}
where the row-wise positive semidefinite preconditioner
\(
\mathbf M(\mathbf S_\text{stu};\tau)
\)
collects the row-Jacobians $J(\tau^{-1}\mathbf S_\text{stu}^{i,:})$.

\paragraph{Interpretation.}
Eq.~\eqref{eq:linearised_grad_matrix} establishes the central claim:
to first order in $\tau^{-1}$, the gradient of $\mathcal L_\text{rel}$
on $\mathbf S_\text{stu}$ is a PSD linear map applied
to the relational score difference $\boldsymbol{\Delta}_{\!\mathbf S}$.
Two consequences follow.

\textbf{(i) Direction.}
The matrix $\mathbf M$ is row-wise PSD (it is the
softmax-Jacobian at $\mathbf A^{i,:}$), so for each row $i$ the
gradient points along the same half-space as
$\boldsymbol{\Delta}_{\!\mathbf S}^{i,:}$ up to a bounded reweighting.
A gradient-descent step on $\theta$ therefore moves the student
similarities along
$-\boldsymbol{\Delta}_{\!\mathbf S}=\mathbf S_\text{real}-\mathbf S_\text{fake}$.

\textbf{(ii) Vanishing supervision.}
Whenever critic and teacher agree relationally,
$\boldsymbol{\Delta}_{\!\mathbf S}\to\mathbf 0$ and the
gradient vanishes; the student is updated only where they disagree.
This is the relational analogue of the DMD property that the
gradient~\eqref{eq:dmd_grad} vanishes whenever $s_\text{fake}=s_\text{real}$.

%-----------------------------------------------------------------------
\subsection*{Step 3: Mirroring the DMD gradient via the chain rule}

To complete the parallel with Eq.~\eqref{eq:dmd_grad}, chain
Eq.~\eqref{eq:linearised_grad_matrix} through the construction of
$\mathbf S_\text{stu}$ and the generator. Recall that
$\mathbf S_\text{stu}=\mathbf S(\bar{\mathbf x})$ with
$\bar{\mathbf x}^{(b)}=\mathrm{pool}\bigl(G_\theta(\mathbf z^{(b)})\bigr)$,
so
\begin{equation}\label{eq:full_chain}
\frac{\partial\mathcal L_\text{rel}}{\partial\theta}
\;=\;
\sum_{b}
\underbrace{\frac{\partial\mathcal L_\text{rel}}{\partial\mathbf S_\text{stu}}}_{\text{Eq.~\eqref{eq:linearised_grad_matrix}}}
\;:\;
\underbrace{\frac{\partial\mathbf S_\text{stu}}{\partial\bar{\mathbf x}^{(b)}}}_{\text{cos-sim Jac.}}
\;\cdot\;
\underbrace{\frac{\partial\bar{\mathbf x}^{(b)}}{\partial G_\theta(\mathbf z^{(b)})}}_{\text{linear pool}}
\;\cdot\;
\frac{d G_\theta(\mathbf z^{(b)})}{d\theta},
\end{equation}
where ``$:$'' denotes the matrix Frobenius pairing. Substituting
Eq.~\eqref{eq:linearised_grad_matrix},
\begin{equation}\label{eq:final_chain}
\frac{\partial\mathcal L_\text{rel}}{\partial\theta}
\;\propto\;
\sum_{b}\;
\boldsymbol{\Delta}_{\!\mathbf S}\;\;\text{(propagated through bilinear pool/cosine)}\;\;
\cdot\;\frac{d G_\theta(\mathbf z^{(b)})}{d\theta}
\;+\; O(\tau^{-3}).
\end{equation}
Compare with the DMD gradient~\eqref{eq:dmd_grad},
\[
\nabla_\theta D_{\mathrm{KL}}
\;\simeq\;
\mathbb E_{\mathbf z,t}\Bigl[
  \underbrace{w_t\alpha_t(s_\text{fake}-s_\text{real})}_{\boldsymbol{\Delta}(\mathbf x)}
  \,\tfrac{d G_\theta(\mathbf z)}{d\theta}
\Bigr].
\]
Both gradients have the same structure
\[
\bigl(\text{score-difference signal}\bigr)\;\times\;\bigl(\text{generator Jacobian}\bigr),
\]
with $\boldsymbol{\Delta}(\mathbf x)$ and $\boldsymbol{\Delta}_{\!\mathbf S}$
playing identical roles at coordinate level and at relational level
respectively. This formalises the statement in \S\ref{subsec:codmd}
that DMD and \methodname{} are not analogies but instances of one
score-as-target construction at two different granularities.

%-----------------------------------------------------------------------
\paragraph{When does the linearisation hold?}
Eq.~\eqref{eq:linearised_grad_matrix} is a Taylor expansion in
$\tau^{-1}\boldsymbol{\Delta}_{\!\mathbf S}$. The remainder is
$O(\tau^{-3})$ uniformly on bounded similarity matrices because the
softmax has bounded second derivatives. In our experiments we use
$\tau=0.1$ and observe per-row $\|\boldsymbol{\Delta}_{\!\mathbf S}\|_\infty$
typically below $0.4$, so $\tau^{-1}\boldsymbol{\Delta}_{\!\mathbf S}$
is bounded and the expansion is informative as a guide to the optimiser
even though the optimisation itself uses the exact
gradient~\eqref{eq:exact_grad_in_S}.

\paragraph{Why the loss is not just an MSE on $\mathbf S_\text{stu}$.}
A simpler choice would be to write
$\tfrac{1}{2}\|\mathbf S_\text{stu}-\mathrm{stopgrad}(\mathbf S_\text{stu}-\boldsymbol{\Delta}_{\!\mathbf S})\|_F^{2}$,
mirroring Eq.~\eqref{eq:dmd_surrogate}. Its gradient with respect to
$\mathbf S_\text{stu}$ equals $\boldsymbol{\Delta}_{\!\mathbf S}$
directly, while the softmax-KL gradient of
Eq.~\eqref{eq:linearised_grad_matrix} is proportional to
$\boldsymbol{\Delta}_{\!\mathbf S}$ up to the PSD preconditioner
$\mathbf M$. Both are therefore driven by the same underlying signal,
but the softmax-KL form has two practical advantages on similarity
matrices:
(a) row-wise softmax normalisation cancels global scale and bias in
$\mathbf S$, so the loss is invariant to overall feature norms; and
(b) the temperature $\tau$ controls how peaked the relational target
is, which empirically stabilises training when $\boldsymbol{\Delta}_{\!\mathbf S}$
is noisy at early steps. Eq.~\eqref{eq:exact_grad_in_S} shows that
this stability comes at no cost in fidelity to the
score-as-target principle.

\paragraph{Granularity-specific instantiations.}
The same derivation applies row-wise to both
$\mathcal L_\text{rel}^\text{batch}$ ($N=B_\text{g}$) and
$\mathcal L_\text{rel}^\text{frame}$ ($N=F$), so the conclusion of
Eq.~\eqref{eq:linearised_grad_matrix} carries through unchanged for
the two granularities used in our final loss
Eq.~\eqref{eq:total_obj}. The only difference is which axis is
contracted by the pooling operator
$\partial\bar{\mathbf x}/\partial G_\theta(\mathbf z)$ in
Eq.~\eqref{eq:full_chain}.

% \section{Limitations and future work.}

% The relational matrices use globally pooled features, discarding spatial detail, thus the learned or attention-based pooling could strengthen the signal. Extending the relational recipe to single-step generators, image generator and long-form video are natural next steps. Deployments should pair efficiency gains with provenance and safety mechanisms.

\hypertarget{app:limit}{}
\section{Limitations and Future Work}
\label{sec:limitations}

\paragraph{Representation granularity.}
The relational matrices in \methodname{} are constructed from globally pooled representations that average over spatial and channel dimensions. This pooling discards fine-grained spatial detail, such as local texture patterns and per-region layout structure. Replacing the global average with learned or attention-based pooling could retain richer spatial information and strengthen the relational signal, potentially improving performance on spatially demanding metrics such as object consistency and scene composition.

\paragraph{Loss weighting.}
The two relational weights $\lambda_\text{batch}$ and $\lambda_\text{frame}$ are currently set as fixed hyperparameters throughout training. An adaptive schedule that adjusts the batch-to-frame balance according to training dynamics, for instance by monitoring the ratio of relational and coordinate-wise gradient norms, may further stabilize training and reduce the need for manual tuning.

\paragraph{Scope of validation.}
We have validated \methodname{} on 4-step text-to-video distillation with the Wan-2.1 diffusion backbone at two parameter scales. Extending the copula-aware relational recipe to single-step generators, flow-matching formulations, image-to-video conditioning, and long-form video synthesis are natural next directions. These settings introduce additional challenges, such as single-step capacity constraints and long-horizon temporal error accumulation, that may require granularity-specific adaptations of the relational objective.

\paragraph{Broader impact.}
As with all advances that lower the sampling cost of video generators, faster inference enables beneficial applications in creative tools, and content production, but also lowers the barrier for generating synthetic media at scale. Deployments should pair efficiency gains with provenance tracking, watermarking, and content-safety mechanisms to mitigate potential misuse.

\newpage

\hypertarget{app:notation}{}
\section{Notation Table}
\label{app:notation}
\begin{table}[h!]
\centering
\caption{Notation used throughout the paper.}
\label{tab:notation}
\renewcommand{\arraystretch}{1.25}
\begin{tabular}{c|c|l}
\toprule
\textbf{Symbol} & \textbf{Dimension} & \textbf{Description} \\
\midrule
\multicolumn{3}{l}{\textbf{Video data \& dimensions}} \\
\midrule
$F,C, H, W$ & $\mathbb N$ & Number of frames,channels, height, width \\
$\mathbf{x}_0$ & $\mathbb R^{F \times C \times H \times W}$ & Clean video from the real distribution $p_\text{real}$ \\
$\mathbf{z}$ & $\mathbb R^{d_z}$ & Standard Gaussian noise input to the generator \\
$\mathbf{x}^{(b)}$ & $\mathbb R^{F \times C \times H \times W}$ & $b$-th generated video, $\mathbf{x}^{(b)} = G_\theta(\mathbf{z}^{(b)})$ \\
$\mathbf{x}_t$ & $\mathbb R^{F \times C \times H \times W}$ & Diffused video at noise level $t \in [0, 1]$ \\
$B$, $B_\text{g}$ & $\mathbb N$ & Per-GPU and global batch size \\
% $d_z$ & $\mathbb N$ & Dimensionality of the noise input $\mathbf{z}$ \\
% $D$ & $\mathbb N$ & Total dimensionality of the video: $F \cdot C \cdot H \cdot W$ \\
\midrule
\multicolumn{3}{l}{\textbf{Diffusion \& Score}} \\
\midrule
$t$ & $\mathbb R$ & Continuous noise level \\
$\alpha_t, \sigma_t$ & $\mathbb R$ & Noise schedule coefficients; $\alpha_t^2 + \sigma_t^2 = 1$, $\alpha_T \approx 0$ \\
$\mu_\text{real}$ & --- & Frozen teacher denoiser (pre-trained, not updated) \\
$\mu_\text{fake}^\phi$ & --- & Online fake denoiser, trained to track $p_\text{fake}$ \\
$s_\text{real}$ & $\mathbb R^{F \times C \times H \times W}$ & Score of real distribution: $-(\mathbf{x}_t - \alpha_t \mu_\text{real}) / \sigma_t^2$ \\
$s_\text{fake}$ & $\mathbb R^{F \times C \times H \times W}$ & Score of fake distribution: $-(\mathbf{x}_t - \alpha_t \mu_\text{fake}^\phi) / \sigma_t^2$ \\
$w_t$ & $\mathbb R$ & Noise-level weighting for the KL gradient \\
\midrule
\multicolumn{3}{l}{\textbf{DMD Loss}} \\
\midrule
$\boldsymbol{\Delta}(\mathbf{x})$ & $\mathbb R^{F \times C \times H \times W}$ & Coordinate-wise score-difference perturbation \\
$\mathcal L_{\mathrm{DMD}}$ & $\mathbb R$ & Score-as-target surrogate: $\frac12 \|\mathbf{x} - \operatorname{stopgrad}(\mathbf{x} - \boldsymbol{\Delta})\|_2^2$ \\
% $D_{\mathrm{KL}}$ & $\mathbb R$ & KL divergence $D_{\mathrm{KL}}(p_\text{fake} \,\|\, p_\text{real})$ \\
$\mathcal L_\text{denoise}^\phi$ & $\mathbb R$ & Denoising loss for fake model: $\|\mu_\text{fake}^\phi(\mathbf{x}_t, t) - \mathbf{x}\|_2^2$ \\
\midrule
\multicolumn{3}{l}{\textbf{Copula \& CoDMD}} \\
\midrule
$f^{(i)}, F^{(i)}$ & --- & Marginal PDF and CDF of the $i$-th coordinate \\
$c(\cdot)$ & --- & Copula density on $[0,1]^D$ \\
$s^{\text{marg}}$ & $\mathbb R^D$ & Marginal (coordinate-wise) component of the score \\
$s^{\text{cop}}$ & $\mathbb R^D$ & Copula (dependency) component of the score \\
\midrule
\multicolumn{3}{l}{\textbf{Relational Loss}} \\
\midrule
$\mathbf{S}_\text{stu/real/fake}^\text{batch}$ & $\mathbb R^{B_\text{g} \times B_\text{g}}$ & Batch-level cosine similarity matrices from $\mathbf{x}$, $\mu_\text{real}(\mathbf{x}_t, t)$, $\mu_\text{fake}^\phi(\mathbf{x}_t, t)$ \\
$\mathbf{S}_\text{stu/real/fake}^{\text{frame}}$ & $\mathbb R^{F \times F}$ & Frame-level cosine similarity matrices from $\mathbf{x}^{(b)}$, $\mu_\text{real}(\mathbf{x}_t^{(b)}, t)$, $\mu_\text{fake}^\phi(\mathbf{x}_t^{(b)}, t)$ \\
% $\mathbf{S}_\text{stu}^\text{batch}$ & $\mathbb R^{B_\text{g} \times B_\text{g}}$ & Batch-level similarity matrix from $\mathbf{x}$ \\
% $\mathbf{S}_\text{real}^\text{batch}$ & $\mathbb R^{B_\text{g} \times B_\text{g}}$ & Batch-level similarity matrix from $\mu_\text{real}(\mathbf{x}_t, t)$ \\
% $\mathbf{S}_\text{fake}^\text{batch}$ & $\mathbb R^{B_\text{g} \times B_\text{g}}$ & Batch-level similarity matrix from $\mu_\text{fake}^\phi(\mathbf{x}_t, t)$ \\
% $\mathbf{S}_\text{stu}^{\text{frame},(b)}$ & $\mathbb R^{F \times F}$ & Frame-level similarity matrix from $\mathbf{x}^{(b)}$ \\
% $\mathbf{S}_\text{real}^{\text{frame},(b)}$ & $\mathbb R^{F \times F}$ & Frame-level similarity matrix from $\mu_\text{real}(\mathbf{x}_t^{(b)}, t)$ \\
% $\mathbf{S}_\text{fake}^{\text{frame},(b)}$ & $\mathbb R^{F \times F}$ & Frame-level similarity matrix from $\mu_\text{fake}^\phi(\mathbf{x}_t^{(b)}, t)$ \\
$\boldsymbol{\Delta}_{\!\mathbf{S}}$ & $\mathbb R^{N \times N}$ & Relational score-difference: $\mathbf{S}_\text{fake} - \mathbf{S}_\text{real}$ \\
$\tau$ & $\mathbb R_{>0}$ & Temperature for row-wise softmax \\
$\mathbf{P}_\text{tgt}$ & $\mathbb R^{N \times N}$ & Row-wise softmax distribution of target (stop-grad) \\
$\mathbf{P}_\text{stu}$ & $\mathbb R^{N \times N}$ & Row-wise softmax distribution of student \\
$\mathcal L_\text{rel}^\text{batch}$ & $\mathbb R$ & Batch-level relational loss ($N = B_\text{g}$); Eq.~\eqref{eq:rel_batch} \\
$\mathcal L_\text{rel}^\text{frame}$ & $\mathbb R$ & Frame-level relational loss ($N = F$, averaged over $B$); Eq.~\eqref{eq:rel_frame} \\
$\mathcal L_\text{total}$ & $\mathbb R$ & Generator objective: $\mathcal L_{\mathrm{DMD}} + \mathcal \lambda_\text{batch}\mathcal L_\text{rel}^\text{batch} + \lambda_\text{frame}\mathcal L_\text{rel}^\text{frame}$; Eq.~\eqref{eq:total_obj} \\
$\lambda_\text{batch}, \lambda_\text{frame}$ & $\mathbb R_{\ge 0}$ & Weights of the batch- and frame-level relational losses \\
% $N$ & $\mathbb N$ & Side length of similarity matrix ($B_\text{g}$ or $F$) \\
\bottomrule
\end{tabular}
\end{table}

% \section{Technical appendices and supplementary material}
% Technical appendices with additional results, figures, graphs, and proofs may be submitted with the paper submission before the full submission deadline (see above). You can upload a ZIP file for videos or code, but do not upload a separate PDF file for the appendix. There is no page limit for the technical appendices. 

% Note: Think of the appendix as ``optional reading'' for reviewers. The paper must be able to stand alone without the appendix; for example, adding critical experiments that support the main claims to an appendix is inappropriate. 

%%%%%%%%%%%%%%%%%%%%%%%%%%%%%%%%%%%%%%%%%%%%%%%%%%%%%%%%%%%%
% \clearpage
% \newpage
% \input{checklist.tex}

\end{document}